\pdfoutput=1
\documentclass[11pt]{article}

\usepackage[margin=1in]{geometry}

\usepackage[T1]{fontenc}
\usepackage[utf8]{inputenc}
\usepackage{microtype}

\usepackage{mathtools,amssymb,amsthm}
\usepackage{xcolor}
\usepackage{tikz}
\usepackage{enumitem}

\usepackage[numbers]{natbib}

\usepackage{hyperref}
\hypersetup{
    colorlinks,
    linkcolor={red!50!black},
    citecolor={blue!50!black},
    urlcolor={blue!80!black}
}
\usepackage[nameinlink,noabbrev]{cleveref}

\newcommand{\sasha}[1]{}
\newcommand{\comment}[1]{}

\theoremstyle{plain}
\newtheorem{theorem}{Theorem}[section]
\newtheorem{example}{Example}[section]
\newtheorem{maintheorem}{Main Theorem}
\newtheorem{lemma}[theorem]{Lemma}
\newtheorem{proposition}[theorem]{Proposition}
\newtheorem{corollary}[theorem]{Corollary}

\theoremstyle{definition}
\newtheorem{definition}[theorem]{Definition}

\theoremstyle{remark}
\newtheorem{remark}[theorem]{Remark}

\newenvironment{restatedtheorem}[1]
  {\par\noindent\textbf{Theorem~\ref{#1}.}\itshape}
  {\par}

\newenvironment{restatedlemma}[1]
  {\par\noindent\textbf{Lemma~\ref{#1}.}\itshape}
  {\par}

\newtheorem*{theorem*}{Theorem}
\newtheorem*{lemma*}{Lemma}
\newtheorem*{proposition*}{Proposition}
\newtheorem*{corollary*}{Corollary}
\newtheorem*{claim*}{Claim}
\newtheorem*{definition*}{Definition}
\newtheorem*{assumption*}{Assumption}
\newtheorem*{question*}{Question}
\newtheorem*{remark*}{Remark}
\newtheorem*{example*}{Example}

\DeclareMathOperator{\VCdim}{\mathbb{VC}}

\newcommand{\R}{\mathbb{R}}
\newcommand{\N}{\mathbb{N}}
\newcommand{\Z}{\mathbb{Z}}

\newcommand{\cF}{\mathcal{F}}

\newcommand{\cH}{\mathcal{H}}

\newcommand{\cN}{\mathcal{N}}

\newcommand{\cS}{\mathcal{S}}

\newcommand{\erf}{\operatorname{erf}}
\newcommand{\Poly}{\operatorname{Poly}}
\newcommand{\rPfaff}{\operatorname{rPfaff}}
\newcommand{\rexp}{\operatorname{rexp}}

\title{Strategic PAC Learnability via Geometric Definability}

\author{
Yuval Filmus\textsuperscript{1,2}
\and
Shay Moran\textsuperscript{1,2,3,4}
\and
Elizaveta Nesterova\textsuperscript{1}
\and
Nir Rosenfeld\textsuperscript{2}
\and
Alexander Shlimovich\textsuperscript{1}
\\[1ex]
}

\date{}

\begin{document}
\maketitle

\footnotetext[1]{Faculty of Mathematics, Technion -- Israel Institute of Technology.}
\footnotetext[2]{Faculty of Computer Science, Technion -- Israel Institute of Technology.}
\footnotetext[3]{Faculty of Data and Decision Sciences, Technion -- Israel Institute of Technology.}
\footnotetext[4]{Google Research.}

\begingroup
\renewcommand\thefootnote{}
\footnotetext{
Emails:
\texttt{yuvalfi@cs.technion.ac.il},
\texttt{smoran@technion.ac.il},
\texttt{elizavetan@campus.technion.ac.il},\\
\texttt{nirr@cs.technion.ac.il},
\texttt{ashlimovich@campus.technion.ac.il}
}
\endgroup

\begin{abstract}
Strategic classification studies learning settings in which individuals can modify their features, at a cost, in order to influence the classifier’s decision. A central question is how the sample complexity of the induced (strategic) hypothesis class depends on the complexities of the underlying hypothesis class and the cost structure governing feasible manipulations. 
Prior work has shown that in several natural settings, such as linear classifiers with norm costs, the induced complexity can be controlled.
We begin by showing that such guarantees fail in general --- even in simple cases: there exist hypothesis classes of VC dimension $1$ on the real line such that, even under the simplest interval neighborhoods, the induced class has infinite VC dimension. Thus, strategic behavior can turn an easy learning problem into a non-learnable one.
To overcome this, we introduce structure via a geometric definability assumption: both the hypothesis class and the cost-induced neighborhood relation can be defined by first-order formulas over $\mathbb R_{\mathtt{exp}}$. Intuitively, this means that hypotheses and costs can be described using arithmetic operations, exponentiation, logarithms, and comparisons. This captures a broad range of natural classes and cost functions, including $\ell_p$ distances, Wasserstein distance, and information-theoretic divergences.
Under this assumption, we prove that learnability is preserved, with sample complexity controlled by the complexity of the defining formulas.
\end{abstract}

\section{Introduction}
The field of strategic classification studies learning in settings where individuals can modify their features in response to a classifier, at some cost, in order to obtain a favorable outcome \citep{bruckner2012static,hardt2016strategic}.  
For example, consider a university that designs a rule for admissions based on applicants’ features such as grades, test scores, or extracurricular activities. Once this rule is made public, prospective students may adapt their behavior to meet the requirements --- by studying harder, retaking exams, or investing in private tutoring. While such adaptations may improve their chances of admission, they also change the distribution of inputs seen by the classifier. The university’s true goal, however, is to admit students who will successfully graduate and contribute to society, rather than those who merely optimize for the admission rule.  

This illustrates an inherent tension between systems that aim to make accurate predictions and individuals who benefit from being classified as positive. Applications of this kind arise in many domains, including loan approval, hiring, university admissions, scholarships, social benefit programs, and medical eligibility. The central objective in strategic classification is therefore to learn a classifier that is accurate under such strategic behavior, i.e., achieves strategic robustness.

A fundamental question in learning theory is which hypothesis classes are learnable, and 
{how many samples are needed.}
Despite over a decade of sustained research,
the implications of strategic behavior on learnability and sample complexity
remain poorly understood.
This gap persists even in simple and well-studied settings,
where the cost function is uniform across individuals,
feature manipulation is unconstrained by causal relationships,
and the classifier is fully known to the agents.

A natural intuition is that if a hypothesis class is learnable and the cost function is ``simple,''
then the induced strategic class should also be learnable,
since local input perturbations are unlikely to cause significant statistical degradation.
Indeed, the literature provides several positive results supporting this view:
linear classes with semi-norm costs \citep{zhang2021incentive,sundaram2023pac};
polytope classes, piecewise-linear classes, and classes closed under input scaling with $\ell_p$ costs \citep{trachtenberg2025nonlinear};
and general classes with separable costs \citep{hardt2016strategic}
or discrete manipulation graphs \citep{cohen2024learnability}.

However, these results stop short of resolving the basic general question:
\emph{does learnability of a hypothesis class, together with a well-behaved cost structure, imply learnability under strategic behavior?}

\paragraph{Combinatorial simplicity is not enough.}
Our first result shows that, in full generality, the answer to the above question is negative.
We consider perhaps the simplest possible setting:
a hypothesis class with VC dimension $1$, defined over the real line,
together with a uniform and highly regular cost structure in which each point $x$ can move within distance at most $1$ (i.e., its neighborhood is an interval of radius $1$ around $x$).
Despite this extreme simplicity,
we construct a class for which the induced strategic class is not learnable. 

\emph{This result demonstrates that learnability is not preserved under strategic behavior,
even under minimal combinatorial complexity and very simple cost structures.}

\paragraph{Geometric structure restores learnability.}
The negative result above suggests that strategic learnability requires structure beyond
combinatorial simplicity.
We therefore turn to a broad and well-studied family of learning problems
with \emph{geometric structure}.

By geometric, we mean that the instance space is $\mathbb{R}^d$,
and that both the hypothesis class~$\mathcal{H}$ and the cost function
are described using a fixed collection of basic operations,
such as the arithmetic operations, comparison operators,
and functions such as $\exp$ and $\log$.
This setting captures many standard classes in machine learning,
including linear classifiers, polynomial threshold functions,
intersections of halfspaces,
decision trees with polynomial decision nodes, neural networks with common activation functions,
as well as cost functions based on distances such as $\ell_p$ norms or KL divergence.
\looseness=-1

Within this framework, we focus on a basic and canonical strategic setting:
each point $x$ is associated with a cost-induced neighborhood $N_x$ to which it can move freely,
but cannot move outside of it---i.e., $N_x$ is the set of points to which moving is potentially cost-effective.
Our goal is to understand when the induced strategic class is learnable,
and to characterize its sample complexity.
Despite its simplicity, this setting is already non-trivial,
and serves as a natural platform for developing our approach.
Moreover, we believe the techniques we introduce extend beyond this setting
to more general models of strategic behavior.

A key feature of strategic classification is that it naturally introduces \emph{existential quantification}.
Indeed, for a hypothesis $h$, let $h'$ denote the induced strategic hypothesis.
An input $x$ is labeled positively after adaptation
if and only if there exists a point $x'$ in its neighborhood
that is labeled positively by $h$:
\begin{equation}\label{eq:exists}
 h'(x)=+1 \iff \exists x' \text{ s.t. } (x' \in N_x) \wedge (h(x')=+1).
\end{equation}
Thus, the induced strategic hypotheses are defined by formulas
that include existential quantifiers.
This observation plays a central role in our approach.

To formalize our geometric setting, we model hypotheses and neighborhoods
as being \emph{definable} by first-order formulas over the reals.
As a first intuition, this can be viewed as specifying a procedure that,
given $h \in \mathcal{H}$ and points $x,x'$,
determines whether $h(x)=1$ and whether $x'$ lies in the neighborhood of $x$
using only the allowed operations (i.e., arithmetic operations, comparisons, $\exp$, and $\log$).
However, this intuition is not entirely accurate:
allowing quantifiers extends this framework beyond purely algorithmic descriptions.
For example, a formula of the form $\psi(x)=\exists x'\,\varphi(x,x')$
can express properties that require searching for a witness in an infinite set,
and are therefore not \emph{a priori} constructive---i.e., it is not clear how to compute them.
\looseness=-1

We note that a closely related line of work in learning theory has studied (non-strategic) hypothesis classes
defined using arithmetic operations alone, without $\exp$ and $\log$.
In this more restrictive setting, terms correspond to polynomials,
and tools from real algebraic geometry have been used to analyze their complexity; a partial list of papers includes
\citep{ben1993localization,Goldberg1995,bartlett98,Ben-DavidES02,AlonMY16,bartlett19,Alon23}.
Much of this work focuses on quantifier-free definitions,
a simplification closely tied to the availability of quantifier elimination
in those settings~\cite{tarski1951decision}.
Extending this framework to include exponentiation substantially increases expressiveness,
and requires different tools, which we draw from geometry and model theory.

\subsection{Overview of Main Results}
We now summarize our main contributions.

\begin{itemize}[leftmargin=1em,topsep=0em,itemsep=0.2em]

\item \textbf{Combinatorial simplicity does not suffice.}
Our first result shows that strategic learnability can fail even in the simplest possible settings.
We construct a hypothesis class over $\mathbb{R}$ with VC dimension $1$,
together with a very simple cost structure in which each point $x$ can move within a fixed radius such that the induced strategic class is not PAC-learnable.
In Appendix~\ref{app:learnability_breaks} we present stronger variants of this phenomenon: one in which non-learnability holds for all radii simultaneously, and another in which the neighborhood system itself has VC dimension $1$, and in Appendix~\ref{app:sec:role-of-definability} we give a different example whose logical complexity is low. These demonstrate that low combinatorial complexity and simple costs
are not sufficient to guarantee strategic learnability.

\item \textbf{Qualitative learnability via definability in $\mathbb{R}_{\exp}$.}
Our first positive result shows that learnability is restored under a geometric definability assumption.
Specifically, if both the hypothesis class and the neighborhood relation
are definable using arithmetic operations together with $\exp$ and $\log$
(i.e., within the structure $\mathbb{R}_{\exp}$),
then the induced strategic class is PAC-learnable.
Moreover, the sample complexity depends on the number of parameters
used in the defining formulas.
Our analysis yields refined bounds for \emph{Empirical Risk Minimization} (ERM) learners
that can be sharper than standard VC-based guarantees,
as they are derived directly from growth function bounds
rather than via Sauer--Shelah--Perles Lemma.

\item \textbf{From qualitative to quantitative bounds.}
The above result is inherently qualitative:
it guarantees the existence of finite sample complexity bounds,
and even yields asymptotic ERM learning rates that improve over VC-based guarantees.
However, it does not provide explicit bounds on the constants appearing in the bounds on the sample complexity and VC~dimension.
This limitation stems from the inherent non-constructive nature of the underlying model-theoretic tools,
and motivates the search for settings in which explicit, quantitative bounds can be obtained.

\item \textbf{Quantitative bounds via restricted definability.}
Our next results address this limitation by providing explicit and constructive sample complexity bounds
for large families of strategically defined geometric classes.
We focus on settings in which the defining formulas lie in $\Sigma_1$,
i.e., use only existential quantifiers.
The form in~\Cref{eq:exists} shows that existential formulas naturally define
the strategic transformation.
Existential formulas can model many natural examples,
including neighborhoods based on $\ell_p$ norms, earth mover distance, KL divergence, and related distances.
We provide detailed examples in Appendix~\ref{sec:examples}.

As a warm-up, we first consider the case without exponentiation,
corresponding to semialgebraic definitions.
In this setting, definable sets reduce to Boolean combinations of polynomial inequalities,
and a single application of quantifier elimination suffices
to eliminate the existential quantifier in~\Cref{eq:exists}.
Combining this with classical tools from real algebraic geometry,
we obtain explicit bounds on the sample complexity and VC~dimension.

We then extend this approach to the setting with exponentiation.
Here, polynomials are replaced by a richer class of real-valued functions for which analogous complexity bounds are known.
Using these tools, we derive explicit upper bounds on the VC dimension
and sample complexity of the induced strategic class.

\end{itemize}

\subsection{Organization}
In \Cref{sec:prel} we provide basic preliminaries from PAC learning and strategic classification.
In \Cref{sec:main-results} we state and explain our main results in detail.
We do not assume prior expertise in logic, and aim to present the relevant concepts
in a way that is accessible to the broad learning theory community.
In \Cref{sec:relatedwork} we survey related work. In Appendix~\ref{sec:examples} we give detailed examples of natural hypothesis classes and neighborhood systems,
together with their first-order definitions,
and illustrate how our theorems apply to these settings. 
The remainder of the paper is devoted to proofs.

\section{Preliminaries}
\label{sec:prel}

This section collects a few notions needed for stating our main results.
We focus on sample complexity in the PAC model, properties of empirical
risk minimization (ERM), and basic definitions from strategic PAC learning.

\paragraph{PAC learning and ERM.}
We work in the standard PAC model of binary classification. A hypothesis
class is a set $\mathcal H\subseteq\{0,1\}^{\mathcal X}$. In the
realizable setting, examples of input-label pairs are drawn i.i.d.\ from a distribution~$D$ over
$\mathcal X\times\{0,1\}$ for which there exists $h^\star\in\mathcal H$
with zero error. In the agnostic setting no such assumption is made,
and labels can admit arbitrary conditional distributions.

For a hypothesis $h\in\mathcal H$, write
\(
L_D(h)=\Pr_{(x,y)\sim D}[h(x)\neq y]
\)
for its population error, and for a sample set
$S=((x_1,y_1),\ldots,(x_m,y_m))$, write
\(
L_S(h)=\frac1m\sum_{i=1}^m \mathbf 1\{h(x_i)\neq y_i\}
\)
for its empirical error.

The central combinatorial parameter governing PAC learnability is the VC
dimension. In particular, a class is PAC learnable if and only if it has
finite VC dimension, and the optimal sample complexity is determined, up to
constant factors, by $\VCdim(\mathcal H)$. The prevailing algorithmic
principle in PAC learning is empirical risk minimization: given a sample,
output a hypothesis in $\mathcal H$ minimizing the empirical error.

We distinguish between optimal PAC sample complexity and ERM sample
complexity. Let $\varepsilon,\delta>0$ be the error and confidence parameters. 
The realizable PAC sample complexity of $\mathcal H$ is the
smallest number $m_{\mathrm{PAC}}^{\mathrm{real}}(\varepsilon,\delta)$
such that there exists a learning algorithm which, given at least this many
realizable examples, outputs with probability at least $1-\delta$ over the sample set a
hypothesis of error at most $\varepsilon$. The agnostic PAC sample
complexity $m_{\mathrm{PAC}}^{\mathrm{agn}}(\varepsilon,\delta)$ is defined
similarly, with the guarantee \(L_D(A(S))\le \inf_{h\in\mathcal H}L_D(h)+\varepsilon.\)

Our quantitative main theorem statements give explicit realizable ERM sample-complexity bounds together with bounds on \(\VCdim(\mathcal H)\), which by classical VC theory immediately imply optimal agnostic PAC and agnostic ERM guarantees up to universal constants.

The realizable ERM sample complexity of $\mathcal H$, denoted
$m_{\mathrm{ERM}}^{\mathrm{real}}(\varepsilon,\delta)$, is the smallest
number such that \emph{every} empirical risk minimizer over $\mathcal H$, when
run on at least $m$ realizable examples, returns with probability at least
$1-\delta$ a hypothesis of error at most $\varepsilon$.

The classical VC bounds imply that for $d=\VCdim(\mathcal H)$,
\[
m_{\mathrm{PAC}}^{\mathrm{real}}(\varepsilon,\delta)
=
\Theta\left(
\frac{d+\log(1/\delta)}{\varepsilon}
\right),
\]
and this rate is tight up to constant factors for every class~\cite{vapnik:74,blumer1989learnability,hanneke2016optimal}.
Moreover, every ERM satisfies
\[
m_{\mathrm{ERM}}^{\mathrm{real}}(\varepsilon,\delta)
=
O\left(
\frac{d\log(1/\varepsilon)+\log(1/\delta)}{\varepsilon}
\right),
\]
and this dependence on $\log(1/\varepsilon)$ is unavoidable for \underline{some}
classes~\cite{vapnik:74,Hanneke16a,BousquetHMZ20}.

In the agnostic setting, ERM is optimal up to universal constants and is equal to \(\Theta\left(\frac{d+\log(1/\delta)}{\varepsilon^2}\right),\)
and this rate is tight for every class (\cite{talagrand:94}; see \cite[Theorems 2.14.1 and 2.6.7]{van-der-Vaart:96}).

\paragraph{Growth function.}
In contrast to the VC-based bounds above, our results provide refined
upper bounds on the realizable ERM sample complexity that can be
substantially smaller than the worst-case VC guarantee. These bounds are
expressed in terms of the growth function.
\begin{definition}[Growth function]
Let $\mathcal H\subseteq\{0,1\}^{\mathcal X}$. The growth function of
$\mathcal H$ is
\[
\Pi_{\mathcal H}(m)
=
\max_{x_1,\ldots,x_m\in\mathcal X}
\left|
\left\{
(h(x_1),\ldots,h(x_m)) : h\in\mathcal H
\right\}
\right|.
\]
Equivalently, $\Pi_{\mathcal H}(m)$ is the maximum number of distinct
labelings induced by $\mathcal H$ on a set of $m$ points.
\end{definition}

\begin{lemma}[ERM bound from the growth function]
\label{lem:erm-from-growth}
Let $\mathcal H\subseteq\{0,1\}^{\mathcal X}$, and suppose that
\(\Pi_{\mathcal H}(m)\le C m^k\)
for every $m\ge 1$, where $C\ge 1$ and $k\ge 1$. Then,
\[
m_{\mathrm{ERM}}^{\mathrm{real}}(\varepsilon,\delta)
=
O\left(
\frac{
k\log(k/\varepsilon)+\log C+\log(1/\delta)
}{\varepsilon}
\right).
\]
\end{lemma}

The proof of this lemma is standard. For completeness, we provide a complete proof in Appendix~\ref{app:technical-lemmas}.

The usual VC-based ERM bound follows by bounding the growth function using Sauer's lemma:
if $\VCdim(\mathcal H)=d$, then \(\Pi_{\mathcal H}(m)\le \left(\frac{em}{d}\right)^d\)
for all $m\ge d$, which yields the standard dependence \(O\left(
\frac{d\log(1/\varepsilon)+\log(1/\delta)}{\varepsilon}
\right)\).
However, in several classes considered in this work, Sauer's lemma gives a
loose upper bound on the growth function. The direct growth-based bound
above can therefore yield significantly sharper guarantees.

In particular, when
\(
\Pi_{\mathcal H}(m)\le C m^k,
\)
the dominant term as $\varepsilon\to 0$ is
\[
\frac{k\log(1/\varepsilon)}{\varepsilon},
\]
so the exponent $k$ of the growth function governs the leading asymptotic
dependence of the ERM sample complexity.

\paragraph{Strategic classification.}
We briefly recall some basic definitions from strategic classification that we will use throughout the paper.

Let $h\colon \mathcal X\to\{0,1\}$ be a hypothesis, and let
$\mathcal N = \{N_x \subseteq \mathcal X : x\in\mathcal X\}$ be a
neighborhood system, where $N_x$ represents the set of points to which an
individual located at $x$ may move.
Note we permit non-uniform neighborhoods, i.e., $N_x$ and $N_{x'}$ can be structurally distinct for different $x$ and~$x'$.
We assume by default that
$x\in N_x$ for all $x$, although our results apply even without this assumption.

The induced strategic hypothesis $h^{\mathcal N}$ is defined by
\[
h^{\mathcal N}(x) = 1
\quad \text{iff} \quad
\exists x'\in N_x \text{ such that } h(x')=1.
\]

For a hypothesis class $\mathcal H$, the induced strategic class is
\(
\mathcal H^{\mathcal N}
=
\{\, h^{\mathcal N} : h\in \mathcal H \,\}.
\)

\section{Main results}
\label{sec:main-results}

\subsection{Strategic non-learnability}

Our first result shows that the VC dimension of a strategically induced hypothesis class $\mathcal H^\mathcal N$ can be arbitrarily larger than the VC dimension of $\mathcal H$ -- in fact, unbounded -- even if $\mathcal N$ is natural and simple.

\begin{maintheorem}[Strategic response can destroy learnability]
\label{main:vc-blowup}
For $r > 0$, let $\mathcal{N}_r$ be the neighborhood system given by $x' \in  N_{x} \iff |x - x'| \leq r$.

There exists a hypothesis class $\mathcal{H}$ over $\mathbb R$ such that $\VCdim(\mathcal H) = 1$ but for every $r > 0$, $\VCdim(\mathcal H^{\mathcal N_r}) = \infty$.
\end{maintheorem}


This parallels earlier results showing increased VC dimension due to strategic behavior in related settings,
including
discrete neighborhoods \cite{zhang2021incentive},
non-uniform neighborhoods \cite{sundaram2023pac},
alternative loss functions \cite{zhang2021incentive},
and label improving responses \cite{attias25a}.
We discuss these connections in more detail in Section~\ref{sec:relatedwork}.



The proof uses the same coding idea as \cite[Lemma~2]{montasser19}. 
In fact, Lemma~\ref{thm:pathology-1d} essentially recovers their construction in the setting of interval neighborhoods on \(\mathbb R\), while Theorem~\ref{main:vc-blowup} adapts it to obtain a single class whose strategic VC dimension is infinite for every radius \(r>0\).

We illustrate the main idea of the construction by describing a simpler variant: for any \emph{fixed} $n \in \mathbb{N}$ we construct a hypothesis class $\mathcal{H}$ over $\mathbb{R}$ such that $\VCdim(\mathcal H) = 1$ while $\VCdim(\mathcal H^{\mathcal N_{1/3}}) = n$. 

For each $S \subseteq [n]$, the hypothesis class includes the hypothesis $f_S = \mathbf{1}\!\left[ x \in \{ i + x_S : i \in S \} \right]$, where $x_S = \sum_{j \in S} 10^{-j}$. By construction, the sets $\{x: f_S(x)=1\}$ are disjoint, and so $\VCdim(\mathcal H) = 1$. Conversely, $\mathcal H^{\mathcal N_{1/3}}$ shatters the points $1+\frac{1}{3},\dots,n+\frac{1}{3}$. 
{The choice of radius \(1/3\) is only for illustrative purposes; an arbitrary fixed radius \(r>0\) is obtained by rescaling the real line.} 
See Appendix~\ref{app:learnability_breaks} and Appendix~\ref{app:sec:role-of-definability} for the full proof and more sophisticated constructions.

\medskip

\subsection{Qualitative bounds}
The phenomenon described in Main~Theorem~\ref{main:vc-blowup} disappears if we restrict both $\mathcal{H}$ and $\mathcal{N}$ to be \emph{definable in $\mathbb{R}_{\exp}$}.
Loosely, this means that hypotheses and neighborhood relations can be described
using (parametric) first-order logic formulas in the \emph{language of $\mathbb R_{\exp}$}.

A formula in the language of $\mathbb R_{\exp}$ is a formula in first-order logic which can make use of arbitrary real constants as well as the functions $+,\cdot,\exp$ and the relations $<,=$, in addition to the usual logical connectives and quantifiers. If the formula uses no quantifiers, we say that it is \emph{quantifier-free}. If it uses only existential quantifiers, and these appear only in the beginning, we say that it is \emph{existential} (in logic, such formulas are called $\Sigma_1$~formulas). If the formula doesn't use $\exp$, then it is also a formula in the language of $\mathbb R$.

\begin{definition}[Definable hypothesis class]
A hypothesis class $\mathcal{H}$ over $\mathbb R^l$ is \emph{definable in $\mathbb{R}_{\exp}$ with $k$ parameters} if there is a formula $\Phi_\mathcal{H}(x,a)$ in the language of $\mathbb{R}_{\exp}$,
where $x=(x_1,\dots,x_l)$ and $a=(a_1,\dots,a_k)$, such that
\[
\mathcal{H}
=
\{ h_a : a\in\mathbb R^k\}, \text{ where } 
h_a
=
\mathbf{1}{\!\left\{ x \in \R^l: \Phi_\mathcal{H}(x,a)\right\}}.
\]

\end{definition}
That is, $\mathcal{H}$ includes all hypotheses $h_a$ corresponding to some parameterization $a$,
where $h_a(x)=1$ iff the formula $\Phi_\mathcal{H}(x,a)$ is satisfied.

As a simple concrete example, consider the hypothesis class of all indicators of halfspaces in $\mathbb R^l$.
This hypothesis class is definable in $\mathbb{R}_{\exp}$ with $k = l + 1$ parameters, using the following formula:
\[
 \Phi_{\text{Halfspace}}(x,a) = a_1 x_1 + \cdots + a_l x_l \ge a_{l+1}.
\]
For any choice of the parameters $a_1,\dots,a_{l+1}$, 
the corresponding hypothesis
\(h_a=\mathbf 1{\!\left\{x\in\mathbb R^l : \Phi_{\text{Halfspace}}(x,a)\right\}}\)
is the indicator function of a closed halfspace. As we vary $a$ over all of~$\mathbb{R}^{l+1}$, we obtain all such hypotheses.
\begin{definition}[Definable neighborhood system]
A neighborhood system $\mathcal{N} = \{ N_x : x \in \mathbb R^l \}$ over $\mathbb R^l$ is \emph{definable in $\mathbb{R}_{\exp}$} if there is a \emph{neighborhood relation} formula $\Phi_\mathcal{N}(x,x')$ (where each of $x,x'$ is a block of $l$ variables) in the language of $\mathbb R_{\exp}$ such that
\[ N_x = \{ x' \in \mathbb R^l : \Phi_\mathcal{N}(x,x') \}. \]
\end{definition}

For example, let $\mathcal{N}_p$ (for $p>0$) be the neighborhood system in
which $N_x$ consists of all points in the $\ell_p$ unit ball around $x$.
This neighborhood system is definable in $\mathbb R_{\exp}$ using the formula
\begin{multline}
\label{eq:phi-np}
\Phi_{\mathcal{N}_p}(x,x')
\equiv
\exists y\in \mathbb{R}^l,z \in \mathbb{R}^l\;
\Bigl[
y_1+\cdots+y_l \le 1
\;\land \\
\bigwedge_{i=1}^l
\Bigl(
(x_i=x_i' \land y_i=0)
\;\lor\;
\bigl(
(\exp z_i=x_i-x_i' \lor \exp z_i=x_i'-x_i)
\land
y_i=\exp(pz_i)
\bigr)
\Bigr)
\Bigr].
\end{multline}
Here $z_i = \log |x_i - x'_i|$ and $y_i = |x_i-x'_i|^p = \exp(p \log |x_i-x'_i|)$.
When $p=2$, \Cref{eq:phi-np} reduces to a simple quadratic formula:
$\Phi_{\mathcal{N}_2}(x,x') \equiv \sum_{i=1}^l (x_i - x'_i)^2 \le 1$.

We can now state our second main theorem, which shows that strategic classes $\mathcal H^\mathcal N$ are learnable if both $\mathcal H$ and $\mathcal N$ are definable in $\mathbb{R}_{\exp}$.

\begin{maintheorem}[Strategic classes definable in $\mathbb{R}_{\exp}$ are learnable]
\label{main:R_exp}
Let $\mathcal H$ be a hypothesis class over~$\mathbb R^l$ definable in $\mathbb{R}_{\exp}$ with $k$ parameters, and let $\mathcal N$ be a neighborhood system definable in~$\mathbb{R}_{\exp}$. Then the strategic class $\mathcal H^\mathcal N$ is PAC learnable.

Moreover, the ERM sample complexity of $\mathcal H^\mathcal N$ satisfies
\[
m_{\mathrm{ERM}}^{\mathrm{real}}(\varepsilon,\delta)
=
O\Bigl(
\frac{k\log(1/\varepsilon) + \log(1/\delta) + K_{\mathcal{H},\mathcal{N}}}
{\varepsilon}
\Bigr),
\]
where $K_{\mathcal H,\mathcal N}$ is a constant depending on $\mathcal H$ and $\mathcal N$.
\end{maintheorem}

Main~Theorem~\ref{main:R_exp} applies to a rich collection of natural strategic settings, including:
\begin{itemize}[leftmargin=1em,topsep=0em,itemsep=0.2em]
\item Polynomial classifiers with various neighborhood systems such as $\ell_p$ balls for $1 \leq p \leq \infty$, where the radius of the ball can depend on the inpu $x$.
\item Hypothesis classes for inputs that are categorical distributions (i.e., over a finite set), with neighborhood systems such as $\ell_p$ balls, KL divergence, and earth mover distance.
\end{itemize}
We elaborate on these examples and others in Appendix~\ref{sec:examples}. 

The proof relies on the observation that the strategic class $\mathcal H^\mathcal N$ is defined by the formula
\[
 \Phi_{\mathcal{H}^\mathcal{N}}(x,a) = \exists y \,\, \Phi_\mathcal N(x,y) \land \Phi_{\mathcal{H}}(y,a),
\]
where $\Phi_\mathcal H,\Phi_\mathcal N$ are the formulas defining $\mathcal H,\mathcal N$, respectively. Given this observation, Main~Theorem~\ref{main:R_exp} follows by applying a result of Johnson and Laskowski~\cite{LaskowskiCompression} which bounds the growth function of hypothesis classes defined in an o-minimal structure such as $\mathbb R_{\exp}$ (see Appendix~\ref{app:formal-definitions} for the definition of o-minimality), and applying Lemma~\ref{lem:erm-from-growth}. The details are given in Appendix~\ref{app:sec:proof-Rexp-without-bounds}.

\subsection{Quantitative bounds}
Notice that as $\varepsilon \to 0$, the constant $K_{\mathcal H, \mathcal N}$ in Main~Theorem~\ref{main:R_exp}
becomes a lower-order term in the sample complexity.
However, for constant $\varepsilon$, it becomes relevant, though it remains unspecified.
This limitation of Main~Theorem~\ref{main:R_exp}
originates in Wilkie's proof of the o-minimality of $\mathbb R_{\exp}$~\cite{Wilkie1996modelcompleteness}, a deep and inherently nonconstructive result. 

If instead of $\mathbb R_{\exp}$ we restrict ourselves to quantifier-free formulas in the language of $\mathbb R$ (that is, not using $\exp$), then classical results in real algebraic geometry, which have been used in several prior works in learning theory (for example~\cite{ben1993localization,Goldberg1995,Ben-DavidES02,AlonMY16,Alon23}), enable us to obtain completely explicit bounds. 

\begin{maintheorem}[Quantitative learnability of strategic classes definable in quantifier-free $\mathbb R$]
\label{main:bounds_R}
Let~$\mathcal{H}$ be a hypothesis class over $\mathbb R^l$ definable with $k$ parameters in the language of $\mathbb{R}$, and let $\mathcal N$ be a neighborhood system definable in the same language. Suppose that both $\mathcal H$ and $\mathcal N$ are defined by quantifier-free formulas involving polynomial equalities and inequalities whose \emph{total description complexity} is at most $D$ (that is, the formulas involve at most $s$ polynomials, each of degree at most~$D'$, with $D = sD'$). 
Then
\[
  m_{\mathrm{ERM}}^{\mathrm{real}}(\varepsilon,\delta)
  =
  O\Bigl(
    \frac{
      k\log(1/\varepsilon)
      + k^2 l \log D
      + \log(1/\delta)
    }{\varepsilon}
  \Bigr),
\]
and
\[
  \VCdim(\mathcal H^\mathcal N)
  =
  O\bigl(k^2 l \log D\bigr).
\]
\end{maintheorem}

If a hypothesis class is defined using a quantifier-free formula, then Warren-type sign-pattern bounds in the style of~\cite{Goldberg1995} give the desired bounds on the growth function, and so on the sample complexity, via Lemma~\ref{lem:erm-from-growth}, as in the preceding main theorem. In our case, $\mathcal H^\mathcal N$ is given by the existential formula
\[
 \Phi_{\mathcal{H}^\mathcal{N}}(x,a) = \exists y \, \bigl(\Phi_\mathcal N(x,y) \land \Phi_{\mathcal{H}}(y,a)\bigr).
\]
Fortunately, in $\mathbb R$ we have \emph{quantifier elimination}, a deep result of Tarski~\cite{tarski1951decision} which states that every first-order formula over $\mathbb R$ is equivalent to one which is quantifier-free. We complete the proof of Main~Theorem~\ref{main:bounds_R} using modern quantitative quantifier elimination results~\cite[Chapter 14]{basu2006}.
The procedure is constructive, enabling to derive explicit bounds.
Details are given in Appendix~\ref{app:sec:proof-R-bounds}.

While Main~Theorem~\ref{main:bounds_R} covers many situations of interest, it falls short of capturing important cases such as standard metric and similarity measures used in practice, including earth mover distance (which requires existential formulas), KL divergence, and $\ell_p$-balls for irrational $p$. Addressing these requires the richer language of $\mathbb R_{\exp}$. This setting is covered by our final main theorem. The theorem involves two complexity parameters, \emph{format} and \emph{degree}, which we define after stating the theorem.

\begin{maintheorem}[Quantitative learnability of strategic classes definable in existential $\mathbb R_{\exp}$]
\label{main:bounds_R_exp}
Let~$\mathcal H$ be a hypothesis class over $\mathbb R^l$ definable with $k$ parameters in the language of $\mathbb R_{\exp}$, and let $\mathcal N$ be a neighborhood system definable in the language of $\mathbb R_{\exp}$. Suppose, moreover, that both $\mathcal H$ and $\mathcal N$ are defined using existential formulas of format $F$ and degree $D$. Then
\[
  m_{\mathrm{ERM}}^{\mathrm{real}}(\varepsilon,\delta)=  O\Bigl(\frac{k\log(1/\varepsilon) + \gamma(F) \log(D) + \log(1/\delta)}{\varepsilon}\Bigr)
\]
and
\[
\VCdim(\mathcal H^\mathcal N)  = O_F(\log D),
\]
where $\gamma\colon \mathbb N \to \mathbb N$ is an explicitly computable function.
\end{maintheorem}

The format and degree are defined for formulas in the language of $\mathbb R_{\exp}$ in the following normal form, for blocks of variables $X,Y$:
\[
 \Phi(X) = \exists Y \, \theta(X,Y),
\]
where $\theta$ is a quantifier-free formula which is a logical combination of atomic predicates of one of the forms
\[
 P(X,Y) \ge 0 \quad \text{or} \quad z = \exp w,
\]
where $P$ is an arbitrary polynomial in the variables $X,Y$, and $z,w$ are variables, each belonging to one of the blocks $X,Y$.

As an example, the definition above in \Cref{eq:phi-np} of the neighborhood system $\mathcal{N}_p$ is almost of this form (we need to replace $x_i - x'_i$, $x'_i - x_i$, and $p\cdot z_i$ with new variables $s_i,t_i,u_i$ belonging to the $Y$ block, and add the constraints $s_i = x_i - x'_i$, $t_i = x'_i - x_i$, and $u_i = p \cdot z_i$).

The \emph{format} of $\Phi$ is the sum of three terms: (i) the number of variables in the block $X$, (ii) the number of variables in the block $Y$, (iii) the number of atomic predicates of the form $z = \exp w$.

The \emph{degree} of $\Phi$ is the sum of two terms: (i) the sum of the degrees of all polynomials $P(X,Y)$ occurring as atomic formulas, (ii) the number of atomic predicates of the form $z = \exp w$.

Compared to Main~Theorem~\ref{main:R_exp}, here we are restricted only to existential formulas. This is not a major restriction in practice: e.g.\ all examples in Appendix~\ref{sec:examples-representative} are of this form. 

The proof of Main~Theorem~\ref{main:bounds_R_exp} uses recent quantitative results in tame topology~\cite{binyamini2020effective,binyamini2022sharply}.\footnote{Grothendieck~\cite{Gro} defined tame topology (``topologie mod\'er\'ee'') as a topology which avoids pathological situations such as space-filling curves. The modern usage refers more specifically to the study of o-minimal structures~\cite{vandenDries98}. These are structures such as $\R_{\exp}$ whose definable sets exhibit many of the properties of semialgebraic sets; see Appendix~\ref{app:formal-definitions}.}
The basic idea is to use the approach of~\cite{LaskowskiCompression}, which is also used in the proof of Main~Theorem~\ref{main:R_exp}. In order to obtain explicit bounds, we rely on the effective cell decomposition theorem for $\R_{\rexp}$, where $\rexp$ denotes restricted exponentiation, as established in~\cite{binyamini2020effective,binyamini2022sharply}.
The language of $\R_{\rexp}$ extends $\R$ with the function $\exp|_{[-M,M]}$, which equals $\exp$ for inputs in $[-M,M]$ and is zero otherwise. The proof involves converting the defining formulas of $\mathcal H$ and $\mathcal N$ from $\R_{\exp}$ to $\R_{\rexp}$, a step that we only know how to carry out for existential formulas. See Appendix~\ref{app:sec:proof-main:bounds_R_exp} for the complete proof.

The results of~\cite{binyamini2020effective,binyamini2022sharply} extend more generally to sharply o-minimal structures, a notion defined in these papers. Consequently, Main~Theorem~\ref{main:bounds_R_exp} extends beyond $\R_{\exp}$, allowing the usage of functions like $\erf$; see Appendix~\ref{app:beyond-Rexp} for more details.

\section{Related work}\label{sec:relatedwork}

\paragraph{Strategic classification.}
Since its introduction in \citep{bruckner2012static,hardt2016strategic},
the field of strategic classification has drawn much attention.
Many works have extended the basic setting to support
causal effects \citep{miller2020strategic,shavit2020causal,horowitz2023causal},
varied sources of uncertainty \citep{ghalme2021strategic,bechavod2022information,shao2023strategic,lechner2023strategic,rosenfeld2024oneshot},
and flexible user incentives \citep{levanon2022generalized,sundaram2023pac}.
Nonetheless, due to the challenges strategic behavior introduces,
most works restrict their focus to linear classifiers and $\ell_2$ costs.
Some advances have been made in establishing strategic learnability and sample complexity -- first for linear classifiers \citep{sundaram2023pac},
and then for broader structured hypothesis classes \citep{cohen2024learnability,trachtenberg2025nonlinear}.

A standard way to formalize the complexity of strategic classifiers is through the induced \emph{effective} class: for each hypothesis, one considers the classifier obtained after individuals best-respond to it, equivalently by composing the original hypothesis with its induced strategic response map \citep[Definition 2]{zhang2021incentive}. Building on this perspective, we express the induced strategic hypotheses as definable existential formulas. This provides a geometric model-theoretic way to study their complexity, and may be useful for future work on strategic learning.



\paragraph{Finite VC dimension of the original class is not enough.}
Several related works show that finite VC dimension of the original class need not control the complexity or learnability of the effective model class induced by user behavior.
Zhang and Conitzer~\cite{zhang2021incentive} study incentive-aware PAC learning through manipulation graphs,
and construct an example using non-uniform graphs (i.e., that can differ across users)
in which the effective VC can blow up arbitrarily.
Lechner and Urner~\cite{lechner2022learning} extend this framework 
and show a similar result for an ad-hoc `strategic' loss which differs from the conventional 0-1 loss.
Sundaram et al.~\cite{sundaram2023pac} study a richer response model in which users 
seek arbitrary (yet observable) predictions, rather than positive only,
and demonstrate VC blow-up using non-uniform (referred to as `instance-wise') neighborhoods.
While bearing similarities to our construction, 
Main~Theorem~\ref{main:vc-blowup} remains within the scope of the original strategic classification setting of Hardt et al.~\cite{hardt2016strategic}:
which neighborhood are continuous, simple, and uniform across users,
and learning aims to minimize the standard 0-1 loss.

Relatedly, Attias et al.~\cite{attias25a} study a setting where manipulating an instance $x$ also changes its true label $y$,
which they refer to as `PAC learning with improvements'.
This differs from the conventional strategic setting capturing gaming behavior in which $y$ remains unchanged.
Interestingly,
their non-learnability example is similar in spirit to our negative result: \cite[Example~2]{attias25a} gives a class of finite VC dimension that is not learnable with improvements, while our Main~Theorem~\ref{main:vc-blowup} gives a class whose ordinary VC dimension is $1$ but whose induced strategic class has infinite VC dimension for every radius $r>0$.
A closely related phenomenon also appears in adversarially robust learning;
from a strategic modeling perspective, this is akin to assuming users seek predictions
which are opposite to their true label---a special case of Sundaram et al.~\cite{sundaram2023pac}.
Here, Montasser, Hanneke, and Srebro~\citep{montasser19} construct a class of VC dimension at most $1$ for which proper robust PAC learning fails under metric-ball perturbations.

\paragraph{Geometric hypothesis classes.} 
There is a long line of work bounding the VC dimension of hypothesis classes
described by polynomial inequalities and related analytic operations.
Goldberg and Jerrum \citep{Goldberg1995} gave general VC-dimension bounds for
hypothesis classes parameterized by real numbers and defined by polynomial
conditions. Related real-algebraic and sign-pattern methods have been used in
learning theory to analyze semialgebraic classes, a partial list includes
\citep{ben1993localization,Ben-DavidES02,AlonMY16,Alon23}. Closely related counting questions have also been studied in combinatorics:
Sauermann~\citep{SAUERMANN2021107593} proved essentially tight bounds for graph families whose edge relations are determined by polynomial sign conditions. Another important line concerns neural networks, where VC-dimension bounds have been obtained for networks with piecewise polynomial and piecewise linear structure~\citep{bartlett98,bartlett19}.

Our semialgebraic bounds follow this tradition, but the strategic transformation introduces an existential quantifier, requiring quantitative
quantifier elimination before sign-pattern bounds can be applied.

\paragraph{Model-theoretic and tame-geometric tools.}
Quantifier elimination over real closed fields goes back to
Tarski~\citep{tarski1951decision} and Seidenberg~\citep{seidenberg1954}.
Effective algorithmic versions were later developed via cylindrical algebraic
decomposition by Collins~\citep{collins1975}, and refined by Renegar
\citep{Renegar1992I,Renegar1992II,Renegar1992III} and by
Basu--Pollack--Roy~\citep{BasuPollackRoy1996,Basu1999}.

Beyond the semialgebraic setting, o-minimality provides a general framework for
tame geometry. Wilkie proved that the real exponential field \(\R_{\exp}\) is
o-minimal~\citep{Wilkie1996modelcompleteness}. A closely related line, rooted
in Khovanskii's theory of fewnomials and Pfaffian functions
\citep{Khovanskii1991Fewnomials}, studies restricted Pfaffian structures:
Wilkie's theorem of the complement established their o-minimality
\citep{Wilkie1999Complement}. Effective bounds for Pfaffian and sub-Pfaffian
sets were developed by Gabrielov and
Vorobjov~\citep{GabrielovVorobjov2004,GabrielovVorobjov2009}. More recently,
Binyamini and Vorobjov obtained quantitative cell decomposition results for
restricted sub-Pfaffian sets
\citep{binyamini2020effective,binyamini2022sharply}, which we use in our
quantitative bounds.


Recent work of \cite{balcan25} studies data-driven algorithm design for parameterized algorithm families whose performance has Pfaffian or piecewise-Pfaffian structure, obtaining learning guarantees in both statistical and online settings. In their online-learning analysis, \cite{balcan25} use the sub-Pfaffian cell-decomposition bounds of  \cite{GABRIELOV2001179} to control the number of regions induced by Pfaffian transition boundaries. 
Our application requires a closely related but slightly different kind of control: in order to derive ERM sample-complexity bounds via growth-function estimates, we need bounds not only on the number of cells in a decomposition, but also bounds on the complexity. For this purpose, the effective cell-decomposition theorem of \cite{binyamini2020effective} is particularly well suited.

These tools have also appeared in learning theory. In particular, Karpinski and Macintyre~\citep{KarpinskiMacintyre1997} used Pfaffian methods to derive VC-dimension bounds for sigmoidal networks and more general Pfaffian neural networks, providing an early connection between Pfaffian geometry and learning theory. In a more general model-theoretic direction, Johnson and Laskowski~\citep{LaskowskiCompression} showed that families definable with \(k\) parameters in an o-minimal structure have polynomially bounded growth function with exponent \(k\). Relatedly, Livni and Simon~\cite{LivniS13} used model-theoretic tools to derive sample compression bounds for definable families.

Our work brings this model-theoretic and tame-geometric perspective to strategic PAC learnability.

\section{Discussion} 
\label{sec:discussion}

\paragraph{Summary of results.}
Our results show that strategic behavior can fundamentally change the
learnability of a hypothesis class: even a VC-dimension-one class on the real
line can become non-learnable after interval-neighborhood expansion. On the
positive side, we show that this pathology disappears under tame geometric
definability assumptions: \(\R_{\exp}\)-definability gives qualitative PAC
learnability, while semialgebraic and existential \(\R_{\exp}\) descriptions
yield explicit sample-complexity and VC-dimension bounds in terms of their
description complexity.

\paragraph{Limitations.}
These positive results should be read within the limits of the tame-definable
framework. The approach excludes neighborhood relations with unbounded
oscillatory or periodic structure, does not give quantitative bounds for
arbitrary quantified \(\R_{\exp}\)-formulas, and does not automatically apply to
distributional neighborhoods defined through integration or infinite-dimensional
optimization. Such distributional models are covered only when the resulting
neighborhood relation admits a controlled finite-dimensional definable
description. We discuss these limitations in
Appendix~\ref{app:limitations}.

\paragraph{Open directions.}
Our framework captures the reachability aspect of strategic behavior: a point is
positive if it can reach some positive point. A natural next step is to enrich
this model by incorporating how agents choose among reachable points, and how
these choices reshape the data distribution. It would be interesting to
understand whether model-theoretic tools can also be used to analyze such richer
response models.

Another direction is to sharpen the quantitative bounds. The dependence on the
format in Main~Theorem~\ref{main:bounds_R_exp} is inherited from general
effective cell decomposition theorems and is unlikely to be optimal for
learning-theoretic applications. Improving these bounds, or obtaining sharper
estimates for special classes of strategic models, would make the theory more
informative.

\section{Examples}

\label{sec:examples}

The theorems in Section~\ref{sec:main-results} apply to a broad collection
of natural strategic settings. In particular, many commonly used hypothesis
classes and neighborhood systems
\(
\mathcal N=\{N_x:x\in\R^d\}
\)
are definable either in \(\R\) or in \(\R_{\exp}\). We now describe several representative examples to which our main theorems apply.

\subsection{Representative examples}
\label{sec:examples-representative}

\paragraph{Halfspaces with neighborhoods depending on the representative point.}

Consider the base class~\(\mathcal H\) of halfspaces in \(\R^2\), defined by
\[
  \Phi_{\text{Halfspace}}(x,a)
  :=
  a_1x_1+a_2x_2+a_3>0,
\]
where \(x\in\R^2\) and \(a\in\R^3\).

Define the neighborhood system \(\mathcal N=\{N_x:x\in\R^2\}\) by
\[
  y\in N_x
  \iff
  \|x-y\|_2 \le r(x),
  \qquad
  r(x_1,x_2)=\max\{x_2,0\}.
\]
Here \(x_2\) can be interpreted as the amount of strategic resources available
to the point: points with larger \(x_2\) can move farther in order to obtain a
positive classification, while points with \(x_2<0\) cannot move at all.

Equivalently, the neighborhood relation is defined by the semialgebraic formula
\[
  \Phi_{\mathcal N}(x,y)
  :=
  \Bigl(x_2\ge 0 \wedge \|x-y\|_2^2\le x_2^2\Bigr)
  \;\vee\;
  \Bigl(x_2<0 \wedge x=y\Bigr).
\]
Thus \(\mathcal N\) is semialgebraic. Therefore
Theorem~\ref{main:bounds_R} applies to the strategically transformed
halfspace class induced by \(\mathcal N\), and gives
\[
  m_{\mathrm{ERM}}^{\mathrm{real}}(\varepsilon,\delta)
  =
  O\left(
    \frac{\log(1/\varepsilon)+\log(1/\delta)}{\varepsilon}
  \right).
\]

\paragraph{Polynomial classifiers with \(\ell_\infty\) neighborhoods.}

Fix a degree bound \(D\ge 1\). Consider the base class \(\mathcal H\) of polynomial
threshold functions in \(\R^l\), defined by
\[
  \Phi_{\mathcal H}(x,\theta)
  :=
  P_\theta(x)>0,
\]
where \(P_\theta\) ranges over all real polynomials in \(l\) variables of degree
at most \(D\), parametrized by their coefficients \(\theta\).

Fix \(r>0\), and define the neighborhood system
\(\mathcal N_\infty=\{N_x:x\in\R^l\}\) by
\[
  y\in N_x
  \iff
  \|x-y\|_\infty \le r .
\]
Equivalently, the neighborhood relation is defined by
\[
  \Phi_{\mathcal N_\infty}(x,y)
  \iff
  \bigwedge_{i=1}^l
  \left(
    -r \le x_i-y_i \le r
  \right).
\]

Both \(\Phi_{\mathcal H}\) and \(\Phi_{\mathcal N_\infty}\) are
quantifier-free semialgebraic formulas:
\(\Phi_{\mathcal H}\) involves one polynomial inequality of degree at most \(D\),
while \(\Phi_{\mathcal N_\infty}\) is given by \(2l\) linear inequalities.

Therefore Theorem~\ref{main:bounds_R} applies to the strategically transformed
class \(\mathcal H^{\mathcal N_\infty}\), whose defining formula is
\[
  \exists y\;
  \Bigl(
    \Phi_{\mathcal N_\infty}(x,y)
    \wedge
    P_\theta(y)>0
  \Bigr).
\]

Here the number of parameters is
\[
k=\binom{l+D}{D},
\]
corresponding to the coefficients of a degree-\(D\) polynomial in \(l\)
variables. The defining formula uses \(2l\) linear inequalities for the
\(\ell_\infty\)-neighborhood, together with one polynomial inequality
\(P_\theta(y)>0\). Thus, in the notation of
Theorem~\ref{main:bounds_R}, one may take \(s=2l+1\).

\paragraph{Decision trees with polynomial decision rules.}

Fix integers \(D_{tree},q\ge 1\). A \emph{decision tree with polynomial
decision rules} is a binary decision tree of depth at most \(D_{tree}\) in
which every internal node \(v\) is labeled by a polynomial inequality
\[
P_v(x)\ge 0,
\]
where \(P_v\) is a real polynomial in \(l\) variables of degree at most \(q\),
and every leaf is labeled by either \(0\) or \(1\). An input \(x\in\R^l\) is
classified by starting at the root and moving left or right according to the
sign of the corresponding polynomial inequality until reaching a leaf.

Our framework applies to the class of all such trees. Indeed, there are only
finitely many tree topologies and leaf labelings of depth at most
\(D_{tree}\), so it suffices to encode each fixed topology by a single
semialgebraic formula.

A complete binary tree of depth \(D_{tree}\) has at most
\[
T=2^{D_{tree}}-1
\]
internal nodes. A polynomial of degree at most \(q\) in \(l\) variables is
determined by
\[
M=\binom{l+q}{q}
\]
coefficients. Thus the entire tree can be parametrized by at most \(TM\) real
parameters.

For every leaf \(\lambda\), there is a unique path from the root of the tree to
\(\lambda\). Along this path, each internal node contributes one polynomial
condition: either \(P_v(x)\ge 0\) if the path follows the right branch, or
\(P_v(x)<0\) if it follows the left branch.

For example, suppose a depth-\(2\) tree first tests
\[
P_1(x)\ge 0,
\]
and then, on the right branch, tests
\[
P_2(x)\ge 0.
\]
Then the inputs reaching the leaf corresponding to “right then left” are exactly
those satisfying
\[
P_1(x)\ge 0
\quad\wedge\quad
P_2(x)<0.
\]

More generally, for every leaf labeled by \(1\), the set of inputs reaching
that leaf is described by a conjunction of polynomial inequalities. Since the
tree outputs \(1\) whenever the input reaches one of the positively labeled
leaves, the set classified positively by the tree is obtained by taking the
disjunction over all such leaves. Hence the decision region of the tree is
defined by a formula of the form
\[
\bigvee_{\lambda\in L_1}
\quad
\bigwedge_{v\in \operatorname{path}(\lambda)}
\bigl(P_v(x)\;\sigma_{v,\lambda}\;0\bigr),
\]
where \(L_1\) is the set of leaves labeled by \(1\), and each
\(\sigma_{v,\lambda}\) denotes either \(\ge \) or \(<\), according to the
branch chosen at node \(v\).

This is a quantifier-free semialgebraic formula. Since there are only finitely
many tree topologies and leaf labelings of depth at most \(D_{tree}\), taking
a finite disjunction over all of them gives a single formula defining the entire
class of depth-\(D_{tree}\) decision trees with polynomial tests of degree at
most \(q\). Consequently, this class is definable in \(\R\) with at most
\[
k
=
(2^{D_{tree}}-1)\binom{l+q}{q}
\]
real parameters, where \(l\) is the input dimension. These parameters
correspond to the coefficients of one degree-\(q\) polynomial at each possible
internal node of the complete depth-\(D_{tree}\) binary tree.

\begin{remark*}
Bounded-complexity semialgebraic threshold classes are handled in the same way.
After fixing the Boolean structure of a formula with \(s\) polynomial equalities
and inequalities of degree at most \(q\), the only parameters are the
coefficients of these polynomials. The resulting positive region is already a
quantifier-free semialgebraic formula, so the same uniform-definability argument
as above applies.
\end{remark*}

\paragraph{Neural networks.}

Feedforward neural networks also fit naturally into our framework. Fix layer
widths
\[
d_0,d_1,\dots,d_L,
\]
where \(d_0=l\) is the input dimension, and let
\[
\sigma\colon \R\to\R
\]
be an activation function. A feedforward network with activation \(\sigma\) is
a function of the form
\[
f_\theta(x)
=
A_L
\sigma\!\Bigl(
A_{L-1}
\sigma\!\bigl(
\cdots
\sigma(A_1x+b_1)
\cdots
\bigr)
+b_{L-1}
\Bigr)
+b_L,
\]
where the matrices \(A_i\) and vectors \(b_i\) are the trainable parameters,
collectively denoted by \(\theta\).

If \(\sigma\) were omitted, the composition above would collapse to a single
affine map, and the resulting class would already be definable in \(\R\).
Thus the definability of the network class is determined entirely by the
definability of the activation function.

Many standard activations are definable in tame structures. For example, the
ReLU activation \(\sigma(t)=\max\{t,0\}\) is semialgebraic, and therefore
definable in \(\R\). The logistic sigmoid
\(\sigma(t)=\frac{1}{1+\exp(-t)}\) is definable in \(\R_{\exp}\). More
generally, activations constructed from arithmetic operations together with
\(\exp\), such as \(\tanh\), softplus, and ELU, are also definable in
\(\R_{\exp}\).

Consequently, for every fixed architecture and every definable activation
function, the corresponding thresholded network class
\[
\mathcal H
=
\{x\mapsto \mathbf 1[f_\theta(x)\ge 0] : \theta\}
\]
is uniformly definable in the same structure. An explicit logical formula encoding the computation of sigmoid neural
networks is given in Example~\ref{app:example-nn-definability}.

\begin{remark}[Composition and definability]
The neural-network example illustrates a general closure principle: definable
maps may be combined by finite algebraic operations and finite composition
without leaving the ambient structure. Thus, once the elementary building
blocks of a model are definable, for instance affine maps and the activation
function, any fixed finite architecture obtained by composing these blocks is
again definable. In this sense, definability is well suited to compositional
model classes such as neural networks.
\end{remark}



\paragraph{\(\ell_p\)-distance neighborhoods.}

\(\ell_p\)-distance neighborhoods are among the standard ways to model bounded
perturbations in feature space. They arise when an individual may modify its
features only within a limited budget, measured in a fixed norm. Different
choices of \(p\) capture different types of constraints: \(p=2\) gives
Euclidean perturbations, \(p=1\) bounds the total coordinate-wise change, and
\(p=\infty\) gives coordinate-wise bounded deviations. The same definability
argument also covers fixed \(0<p<1\), where the \(\ell_p\) quasi-norm
increasingly favors changes supported on few coordinates; in the limit
\(p\to 0\), this resembles a Hamming-type constraint on the number of modified
features.

Formally, for fixed \(p>0\), define the neighborhood system
\(\mathcal N_p=\{N_x:x\in\R^l\}\) by
\[
  y\in N_x
  \iff
  \|x-y\|_p \le 1,
\]
equivalently,
\[
  \sum_{i=1}^l |x_i-y_i|^p \le 1.
\]
The corresponding neighborhood relation is definable in \(\R_{\exp}\). Indeed,
absolute values can be expressed using algebraic constraints, and for \(t>0\)
the map \(t\mapsto t^p\) can be written as
\[
  t^p=\exp(p\log t),
\]
with the case \(t=0\) handled separately. For completeness, we give an explicit
formula \(\Phi_{\mathcal N_p}(x,y)\) defining \(\ell_p\)-balls in
Example~\ref{app:example-lp-definability}. In particular, the formula
\(\Phi_{\mathcal N_p}(x,y)\) has format and degree \(O(l)\).

Consequently, whenever one has a complexity bound for a formula defining the
base hypothesis class \(\mathcal H\), Theorem~\ref{main:bounds_R_exp} gives an
effective ERM sample-complexity bound for the strategically transformed class
\(\mathcal H^{\mathcal N_p}\).

For example, let
\(A=\{\alpha_i\}_{i\in I}\subseteq \N^l\) be a set of multiindices with
\(|A|=O(l)\), and consider sparse polynomial threshold functions of the form
\[
  P_\theta(x)=\sum_{i\in I}\theta_i x^{\alpha_i}.
\]
Such polynomials are commonly called \emph{fewnomials}.

This class has only \(O(l)\) parameters and format \(O(l)\), while its degree
\[
q:=\max_{i\in I}|\alpha_i|
\]
may be arbitrarily large independently of \(l\). Thus
Theorem~\ref{main:bounds_R_exp} applies to the strategically transformed class
\(\mathcal H^{\mathcal N_p}\), giving non-trivial bounds.

In some cases, there may be no effective way to control the format and degree
of a formula defining the hypothesis class \(\mathcal H\). In such cases, one
can instead use the asymptotic bounds provided by
Main Theorem~\ref{main:R_exp}.

\paragraph{KL-divergence neighborhoods on finite domains.}

KL-divergence neighborhoods arise naturally in settings where one models
multiplicative or information-theoretic perturbations of distributions. They
are also standard in robust optimization and distribution-shift models, where
one allows the underlying distribution to vary within a KL ball around a
nominal distribution.

Let \(X=\{a_1,\ldots,a_l\}\) be a finite set. We consider a strategic
classification problem whose domain is the simplex
\[
  \Delta_l
  =
  \left\{
    x\in \R_{\ge 0}^l : \sum_{i=1}^l x_i=1
  \right\}.
\]
A point \(x=(x_1,\ldots,x_l)\) represents a probability distribution on \(X\),
where \(x_i\) is the mass assigned to \(a_i\).

Define the neighborhood system
\(\mathcal N_{\mathrm{KL}}=\{N_x:x\in\Delta_l\}\) as follows.
On the interior of the simplex,
\[
  y\in N_x
  \iff
  \mathrm{KL}(x\|y)\le 1,
\]
where
\[
  \mathrm{KL}(x\|y)
  =
  \sum_{i=1}^l
  x_i\log\left(\frac{x_i}{y_i}\right).
\]

Equivalently, the neighborhood relation is defined by the formula
\[
  \Phi_{\mathcal N_{\mathrm{KL}}}(x,y)
  \iff
  \sum_{i=1}^l
  x_i\log\left(\frac{x_i}{y_i}\right)
  \le 1,
\]
together with the domain restrictions
\[
  x_i>0,\qquad y_i>0,\qquad
  \sum_{i=1}^l x_i=\sum_{i=1}^l y_i=1.
\]

This relation is definable in \(\R_{\exp}\), since the logarithm is definable
from the exponential function and the remaining operations are algebraic. For
completeness, we give an explicit defining formula in
Example~\ref{app:example-kl-definability}.

\paragraph{Earth mover distance (Wasserstein-1) neighborhoods.}

The earth mover distance (EMD), also known as the Wasserstein-1 distance,
measures the cost of transforming one distribution into another by transporting
mass across a ground metric. Unlike \(\ell_p\) or KL distances, it captures the
\emph{geometry} of the domain, penalizing movement of mass proportionally to
the distance it travels. This makes it a natural model in applications
involving structured domains, such as histograms, spatial data, or ordered
categories.

Formally, let \(X=\{a_1,\dots,a_l\}\) be a finite domain equipped with a ground
metric
\[
  \rho\colon X\times X\to\R_{\ge 0},
\]
assumed to be definable in the underlying structure.\footnote{
In the finite case, this simply means that each value \(\rho(i,j)\) is a fixed
real constant appearing in the formula. For example, if \(X=\{1,2,3\}\), one may
specify
\[
  \rho(1,2)=1,\quad \rho(1,3)=2,\quad \rho(2,3)=1,
\]
and these constants are treated as part of the description of the neighborhood.
More generally, one may allow \(\rho\) to be given by a definable function when
working over continuous domains.
}

As in the previous example, the domain is the family of probability
distributions over \(X\), identified with the simplex
\[
  \Delta_l
  =
  \left\{
    x\in \R_{\ge 0}^l : \sum_{i=1}^l x_i=1
  \right\}.
\]

For two distributions \(x,y\in\Delta_l\), the earth mover distance is
\[
  \mathrm{EMD}(x,y)
  :=
  \min_{\gamma \in \Gamma(x,y)}
  \sum_{i,j=1}^l \gamma_{ij}\rho(i,j),
\]
where \(\Gamma(x,y)\) is the set of couplings between \(x\) and \(y\), i.e.,
nonnegative matrices \(\gamma=(\gamma_{ij})\) such that
\[
  \sum_{j=1}^l \gamma_{ij}=x_i
  \quad\text{and}\quad
  \sum_{i=1}^l \gamma_{ij}=y_j.
\]

Define the neighborhood system
\(\mathcal N_{\mathrm{EMD}}=\{N_x:x\in\Delta_l\}\) by
\[
  y\in N_x
  \iff
  \mathrm{EMD}(x,y)\le 1.
\]

Equivalently, the neighborhood relation is defined by
\[
  \Phi_{\mathcal N_{\mathrm{EMD}}}(x,y)
  \iff
  \exists \gamma\in\R_{\ge 0}^{l\times l}
\]
such that
\[
  \sum_{j=1}^l \gamma_{ij}=x_i \ (1\le i\le l),
  \qquad
  \sum_{i=1}^l \gamma_{ij}=y_j \ (1\le j\le l),
  \qquad
  \sum_{i,j=1}^l \gamma_{ij}\rho(i,j)\le 1.
\]

This relation is definable in \(\R\), since the constraints defining feasible
transport plans are linear, and the transportation cost is a linear function
of \(\gamma\). Therefore this neighborhood system can be used in any of the
three main theorems, depending on the hypothesis class being transformed. For
completeness, we provide an explicit defining formula in
Example~\ref{app:example-emd-definability}.

\paragraph{KL-divergence for Gaussian families.}

As in the discrete case above, KL-divergence is defined via an integral,
\[
\mathrm{KL}(p\|q)
=
\int_{\R} p(x)\log\frac{p(x)}{q(x)}\,dx.
\]
At first sight, this falls outside the definable framework, since our setting
is not closed under integration.

However, for Gaussian families, the integral admits a closed-form expression.
For instance, for the location family
\[
p_\mu=\mathcal N(\mu,1),
\]
one has
\[
\mathrm{KL}(p_\mu\|p_{\mu'})
=
\frac{(\mu-\mu')^2}{2}.
\]

More generally, for Gaussians with arbitrary means and variances, the
KL-divergence is given by an explicit expression involving polynomials,
logarithms, and division. Consequently, the neighborhood system
\(\mathcal N_{\mathrm{KL}}\) defined by
\[
  p_{\mu'}\in N_{p_\mu}
  \iff
  \mathrm{KL}(p_\mu\|p_{\mu'})\le 1
\]
is definable in \(\R_{\exp}\).

Thus, despite being defined via an integral, KL-divergence neighborhoods for
Gaussian families fall within the scope of
Theorem~\ref{main:bounds_R_exp}.

\paragraph{Total-variation neighborhoods for Gaussian families.}

In contrast to KL-divergence, total-variation (TV) distance typically does
not admit a closed-form algebraic expression, and therefore provides a
natural example where integration leads outside the definable
\(\R_{\exp}\) setting.

Given a parametric family \(\{p_\theta\}\), define the neighborhood system
\(\mathcal N_{\mathrm{TV}}=\{N_{p_\theta}\}\) by
\[
p_{\theta'}\in N_{p_\theta}
\iff
d_{\mathrm{TV}}(p_\theta,p_{\theta'}) \le \varepsilon,
\]
where
\[
d_{\mathrm{TV}}(p,q)
=
\frac{1}{2}\int_{\R} |p(x)-q(x)|\,dx.
\]

Unlike the KL case, the integral defining TV does not, in general,
simplify to a finite expression in terms of elementary operations.

For instance, for Gaussian location families,
\[
p_\mu = \mathcal N(\mu,1),
\]
one has the explicit formula
\[
d_{\mathrm{TV}}(p_\mu,p_{\mu'})
=
2\Phi\!\left(\frac{|\mu-\mu'|}{2}\right)-1,
\]
where \(\Phi\) is the Gaussian cumulative distribution function. This
function is expressed via the error function
\[
\operatorname{erf}(x)
=
\frac{2}{\sqrt{\pi}}\int_0^x e^{-t^2}\,dt,
\]
which is known not to be definable in \(\R_{\exp}\)
(see, e.g., \cite[Chapter~4.4]{piatkowski2019exponential}).

However, this does not imply that such neighborhoods are intractable.
The key observation is that the error function is a Pfaffian function
(see Definition~\ref{def:pfaff}). Indeed, setting
\[
f_1(x)=e^{-x^2},
\qquad
f_2(x)=\operatorname{erf}(x),
\]
we obtain a Pfaffian chain:
\[
f_1'(x)=-2x f_1(x),
\qquad
f_2'(x)=\frac{2}{\sqrt{\pi}} f_1(x).
\]
Thus \(\operatorname{erf}\) is definable in the restricted Pfaffian
structure \(\R_{\rPfaff}\).

Consequently, after restricting the parameters to a bounded domain,
the neighborhood relation induced by Gaussian TV distance becomes
definable in \(\R_{\rPfaff}\), and the quantitative bounds from
Appendix~\ref{app:beyond-Rexp} apply to this setting. This example
illustrates that while definability is not preserved under integration
in general, many natural statistical distances remain tame after
passing to the Pfaffian setting.

\smallskip

\begin{remark}[Definable dependence of neighborhoods]
The examples above can be further generalized by allowing the neighborhoods
\(N_x\) to depend on \(x\) in a definable way. If this dependence can be
controlled, and if one can bound the format and degree of the resulting
formula \(\Phi_{\cN}(X,Y)\) defining the neighborhood system, then the same
quantitative bounds apply.

More precisely, the parameters defining the neighborhoods may themselves be
given by definable functions of \(x\). For example, one may consider
KL-balls with radius \(\varepsilon=\varepsilon(x)\), \(\ell_p\)-balls where
\(p=p(x)\) varies with \(x\), or earth mover distances with a ground metric
\(\rho=\rho(x)\). As long as these dependencies are definable in the
underlying structure, the resulting neighborhood system
\(\cN=\{N_x:x\in\R^d\}\) remains definable. The corresponding format and
degree can then be bounded by a straightforward semantic argument.

This includes, as a special case, piecewise-defined neighborhood systems
obtained by partitioning the domain into finitely many definable regions and
assigning a different defining formula \(\Phi_{\cN}(X,Y)\) on each region.
\end{remark}

\subsection{Explicit formulas}

\begin{example}
\label{app:example-lp-definability}
If \(p < \infty\) then we can define the neighborhood relation corresponding to \(\|x-x'\|_p\le 1\) in \(\mathbb R^l\) as
\begin{multline*}
 \Phi_{\mathcal N_p}(x,x') \iff
 \exists y_1,\dots,y_l,z_1,\dots,z_l \; \Bigg[
 \sum_{i=1}^l y_i \leq 1 \land \\ \bigwedge_{i=1}^l \bigl((x_i = x_i' \land y_i = 0) \lor (y_i = \exp (p z_i) \land (\exp z_i = x_i - x_i' \lor \exp z_i = x_i' - x_i))\bigr)
 \Bigg]\,.
\end{multline*}

If \(p = \infty\) then the neighborhood relation corresponding to \(\|x - x'\|_\infty \leq 1\) in \(\mathbb R^l\) is given by
\[
 \Phi_{\mathcal N_\infty}(x,x') \iff \bigwedge_{i=1}^l \bigl( x_i - x_i' \le 1 \land x_i' - x_i \le 1 \bigr)\,.
\]
\end{example}

\begin{example}
\label{app:example-kl-definability}
The neighborhood relation corresponding to \(\mathrm{KL}(x\|x') \leq 1\) for distributions on a finite set of size \(l\), represented as points in \(\mathbb R^l\) is
\begin{multline*}
 \Phi_{\mathcal N_{\mathrm{KL}}}(x,x') \iff \exists z_1,\dots,z_l \Bigg[
 \bigwedge_{i=1}^l (x_i \ge 0 \land x_i' \ge 0) \land \sum_{i=1}^l x_i = 1 \land \sum_{i=1}^l x_i' = 1 \land \\ \sum_{i=1}^l x_i z_i \leq 1 \land
 \bigwedge_{i=1}^l \bigl((x_i > 0 \land x_i' > 0 \land x_i' \exp z_i = x_i)\lor(x_i = 0 \land z_i = 0)\bigr)
 \Bigg] \,.
\end{multline*}
\end{example}

\begin{example}
\label{app:example-emd-definability}
The neighborhood relation corresponding to \(\mathrm{EMD}(x,x')\) for distributions on a finite set of size \(l\), represented as points in \(\mathbb R^l\) with respect to a metric \(\rho\colon [l] \times [l] \to \mathbb{R}_+\) (which we think of as constants) is
\begin{multline*}
 \Phi_{\mathcal N_{\mathrm{EMD}}}(x,x') \iff \exists \gamma_{ij}|_{1 \leq i,j \leq l} \Bigg[
 \bigwedge_{i=1}^l (x_i \ge 0 \land x_i' \ge 0) \land \sum_{i=1}^l x_i = 1 \land \sum_{i=1}^l x_i' = 1 \land \\  \bigwedge_{i,j=1}^l \gamma_{ij} \ge 0 \land
 \bigwedge_{i=1}^l \sum_{j=1}^l \gamma_{ij} = x_i \land
 \bigwedge_{j=1}^l \sum_{i=1}^l \gamma_{ij} = x_j' \land
 \sum_{i,j=1}^l \gamma_{ij} \rho(i,j) \leq 1
 \Bigg] \,.
\end{multline*}
\end{example}

\begin{example}
\label{app:example-nn-definability}

Fix a feedforward architecture with layer widths
\[
d_0,d_1,\dots,d_L,
\]
where \(d_0=l\) is the input dimension and $d_L = 1$, and logistic sigmoid activation
\[
\sigma(t)=\frac{1}{1+\exp(-t)}.
\]

A network in this architecture is determined by weight matrices
\(A^{(j)}\in\R^{d_j\times d_{j-1}}\) and bias vectors
\(b^{(j)}\in\R^{d_j}\) for \(1\le j\le L\). Let
\[
\theta=(A^{(1)},b^{(1)},\dots,A^{(L)},b^{(L)})
\]
denote the trainable parameters. Writing \(z^{(0)}:=x\), the network is
computed recursively by
\[
z^{(j)}
=
\sigma\bigl(A^{(j)}z^{(j-1)}+b^{(j)}\bigr),
\qquad
1\le j\le L,
\]
where \(\sigma\) is applied coordinate-wise. The associated thresholded network
class is
\[
\mathcal H
=
\left\{
x\mapsto \mathbf 1[z^{(L)}_1\ge 0]
:
\theta
\right\}.
\]

This class is definable in \(\R_{\exp}\). Indeed, introduce auxiliary variables
\(r_i^{(j)},q_i^{(j)},z_i^{(j)}\), where \(r_i^{(j)}\) is the affine input to
the neuron, \(q_i^{(j)}=\exp(-r_i^{(j)})\), and \(z_i^{(j)}\) is the neuron
output. Denote $z^{(0)}$ by $x$. The network computation is encoded by
\begin{multline*}
\Phi_{\mathrm{NN}}(x,\theta)
\iff
\exists
r_i^{(j)},q_i^{(j)},z_i^{(j)}
\mid_{1\le j\le L,\;1\le i\le d_j}
\Bigg[
\bigwedge_{j=1}^L
\bigwedge_{i=1}^{d_j}
\Bigg(
r_i^{(j)}
=
\sum_{s=1}^{d_{j-1}}
A^{(j)}_{is} z_s^{(j-1)}
+
b_i^{(j)}
\;\land \\[0.3em]
q_i^{(j)}
=
\exp(-r_i^{(j)})
\;\land\;
z_i^{(j)}(1+q_i^{(j)})=1
\Bigg)
\;\land\;
r^{(L)}_1\ge 0
\Bigg].
\end{multline*}

Thus sigmoid neural networks of fixed architecture are uniformly definable in
\(\R_{\exp}\). Matrix-vector notation is used only as shorthand: all equations
expand into coordinate-wise polynomial equations together with applications of
\(\exp\).
\end{example}

\subsection{Limitations of the definability framework}
\label{app:limitations}
\paragraph{Outside the tame-definable regime.}

Our positive results rely on tameness. Although we focus primarily on
hypothesis classes and neighborhood systems definable in \(\R_{\exp}\), the
same approach also applies beyond the specific examples emphasized in the main
text; see the Gaussian examples in Appendix~\ref{sec:examples} and the
extensions discussed in Appendix~\ref{app:beyond-Rexp}. However, not every
natural neighborhood system can be accommodated by this model-theoretic
framework.

For example, define the neighborhood system \(\mathcal N=\{N_x:x\in\R\}\) by
\[
y\in N_x
\Longleftrightarrow
|\sin(y-x)| \le \varepsilon.
\]
One cannot obtain analogues of Main~Theorem~\ref{main:R_exp} simply by adding
\(\sin\) to the language of the real field. The obstruction is not merely
technical. The neighborhood
\[
N_0=\{y:|\sin y|\le\varepsilon\}
\]
is an infinite periodic union of disjoint intervals, repeating the same local
pattern across infinitely many regions. Such behavior is incompatible with the
tameness underlying \(\R_{\exp}\), formalized by o-minimality; see
Appendix~\ref{app:formal-definitions} for the formal definition of o-minimality.

More importantly, in structures that permit this kind of unbounded periodic
repetition, one cannot expect general learnability guarantees for strategic
expansions of learnable hypothesis classes. A closely related mechanism is exploited in
Lemma~\ref{lem:frac}, which yields infinite VC dimension after neighborhood
expansion; see also Remark~\ref{rem:definable-Z-suffices}. Thus the exclusion of
oscillatory or periodic structure is an essential feature of the framework, not
just a limitation of the proof.

\paragraph{Qualitative versus quantitative guarantees in $\R_{\exp}$.}

The quantitative bounds in Theorem~\ref{main:bounds_R_exp} apply to existentially definable formulas, and therefore do not directly cover formulas with
universal quantifiers. This limitation is quantitative rather than qualitative:
Main Theorem~\ref{main:R_exp} still applies to arbitrary first-order definable hypothesis classes and neighborhood systems in $\R_{\exp}$.

For example, robustly defined hypothesis classes of the form
\[
\Phi_{\mathcal H}(x,a)
\Longleftrightarrow
\forall y \;\bigl(|y-x|\le r \;\Rightarrow\; f_a(y)\ge 0\bigr)
\]
are first-order definable whenever \(f_a\) is definable, and hence fall under
the qualitative guarantee. However, they are not covered by our quantitative
existential bounds unless the universal condition can be converted, with
controlled complexity, into an equivalent existential or quantifier-free form.

\paragraph{Distributional and optimization-based neighborhoods.}
Our framework is not closed under operations such as integration or optimization over infinite-dimensional spaces. Consequently, some natural neighborhood systems modeling strategic behavior fall outside the definable regime. This does not mean that all distributional neighborhood systems are excluded: in some special cases, integral-defined quantities admit explicit finite-dimensional descriptions and therefore remain definable; see Appendix~\ref{sec:examples} for Gaussian examples. However, such reductions are not available in general.

For example, consider the neighborhood system \(\mathcal N_{\mathrm{EMD}}\) defined by
\[
p_{\theta'}\in N_{p_\theta}
\Longleftrightarrow
\mathrm{EMD}(p_\theta,p_{\theta'}) \le \varepsilon,
\]
where
\[
\mathrm{EMD}(p,q)
=
\inf_{\gamma\in\Gamma(p,q)}
\int_{\R^d\times\R^d}
\|x-y\|\,d\gamma(x,y),
\]
and \(\Gamma(p,q)\) denotes the set of couplings of \(p\) and \(q\).

Even when the distributions \(p_\theta\) are definable, the induced
neighborhood relation need not be first-order definable in
\(\R_{\exp}\), since the definition involves optimization over an
infinite-dimensional family of transport plans. This should be contrasted
with the finite-domain Earth-Mover example in
Appendix~\ref{sec:examples}, where the transport problem is finite-dimensional and therefore admits a definable formulation.

\section{Technical Background}

\label{app:formal-definitions}

\subsection{VC dimension and growth functions}
\label{app:vc}

Let \(\mathcal H\subseteq \{0,1\}^X\) be a hypothesis class on a domain \(X\).

We freely identify subsets \(A\subseteq X\) with their indicator functions
\(\mathbf 1_A\colon X\to\{0,1\}\). Thus, when convenient, we will use
set-theoretic notation interchangeably with functional notation for hypothesis
classes.

\begin{definition}[Trace and growth function]
Let \(A\subseteq X\) be finite. The \emph{trace} of \(\mathcal H\) on \(A\) is
\[
  \mathcal H|_A
  :=
  \{
    (h(x))_{x\in A}
    :
    h\in\mathcal H
  \}.
\]
The \emph{growth function} of \(\mathcal H\) is
\[
  \Pi_{\mathcal H}(m)
  :=
  \max_{A\subseteq X,\;|A|=m}
  |\mathcal H|_A|.
\]
\end{definition}

\begin{definition}[VC dimension]
A finite set \(A\subseteq X\) is \emph{shattered} by \(\mathcal H\) if
\[
  \mathcal H|_A
  =
  \{0,1\}^A.
\]
The \emph{VC dimension} of \(\mathcal H\) is
\[
  \VCdim(\mathcal H)
  :=
  \sup
  \{
    |A|
    :
    A\subseteq X
    \text{ is shattered by }\mathcal H
  \}.
\]
\end{definition}

The following classical bound relates the growth function to the VC dimension.

\begin{theorem}[Sauer--Shelah--Perles]
\label{thm:ssp}
If \(\VCdim(\mathcal H)=d<\infty\), then for all \(m\ge d\),
\[
  \Pi_{\mathcal H}(m)
  \le
  \sum_{i=0}^d \binom{m}{i}
  \le
  \left(\frac{em}{d}\right)^d.
\]
\end{theorem}
\subsection{O-minimality and tame topology}

The positive results in this paper rely on a form of geometric tameness.
Historically, Grothendieck \cite{Gro} used the term \emph{tame topology}
(\emph{topologie modérée}) to describe a setting in which pathological
topological phenomena are excluded. In modern model theory, one precise and
widely used formalization of this idea is \emph{o-minimality}. Informally,
o-minimal structures are expansions of the real field whose definable sets
behave like semialgebraic sets: they admit finite decompositions into
simple pieces and cannot exhibit infinite oscillation or repeated discrete
patterns.

\begin{definition}[o-minimality]
An expansion \(\mathcal R\) of the ordered real field is called
\emph{o-minimal} if every definable subset of \(\R\) is a finite union of
points and intervals.
\end{definition}

The structures used in this paper are o-minimal. The real field \(\R\) is
o-minimal by Tarski's quantifier-elimination theorem. The exponential real field
\(\R_{\exp}\) is o-minimal by Wilkie's theorem
\cite{Wilkie1996modelcompleteness}. The restricted Pfaffian structure
\(\R_{\rPfaff}\) is o-minimal by Wilkie's theorem of the complement
\cite{Wilkie1999Complement}.

Although the definition only refers to subsets of the real line, o-minimality
has strong higher-dimensional consequences. In particular, definable subsets of
\(\R^n\) admit finite cell decompositions, and definable families have
controlled combinatorial complexity. See \cite{vandenDries98} for the
general theory.

\begin{definition}[Cylindrical cell decomposition]
\label{def:cell-decomposition}
A \emph{cell} is defined inductively on the ambient dimension. In \(\R\), cells
are points and open intervals. Suppose cells in \(\R^{n-1}\) have already been
defined. A cell in \(\R^n\) is either

\begin{enumerate}
    \item the graph of a continuous definable function
    \[
      f\colon C\to \R
    \]
    over a cell \(C\subseteq \R^{n-1}\), or

    \item a band of the form
    \[
      \{(x,t)\in C\times \R : f(x)<t<g(x)\},
    \]
    where \(C\subseteq \R^{n-1}\) is a cell and
    \(f,g\colon C\to\R\) are continuous definable functions with \(f<g\).
    We also allow \(f=-\infty\) or \(g=+\infty\).
\end{enumerate}

A \emph{cylindrical cell decomposition} of \(\R^n\) is a finite partition of \(\R^n\) into
cells. It is called \emph{compatible} with a set \(X\subseteq \R^n\) if every
cell is either contained in \(X\) or disjoint from \(X\). More generally, it is
compatible with a finite family \(X_1,\ldots,X_s\subseteq \R^n\) if it is
compatible with each \(X_i\).
\end{definition}

The cell decomposition theorem for o-minimal structures states that if
\(X_1,\ldots,X_s\subseteq \R^n\) are definable sets, then there exists a finite
cell decomposition of \(\R^n\) compatible with all of them.

\subsection{VC bounds for polynomially definable classes}

\begin{theorem}[Goldberg--Jerrum bound~{\cite[Theorem~2.2]{Goldberg1995}}]
\label{thm:GJ-param}
Let $\mathcal H = \{h_a : a \in \R^k\}$ be a hypothesis class on $\R^l$.
Assume that membership $x \in h_a$ can be expressed by a quantifier-free
Boolean formula
\[
\Phi(x,a)
\]
involving at most $s$ distinct polynomial (in)equalities in the variables
$(x,a) \in \R^{l+k}$, each of degree at most $D$.

Then
\[
\VCdim(\mathcal H)
\;\le\;
O\!\left(k \log(sD)\right).
\]
\end{theorem}

\subsection{Quantifier elimination}

We use the following quantitative form of one-block quantifier elimination
(see Basu--Pollack--Roy~\cite[Theorem~14.16]{basu2006}).

\begin{theorem}[One-block quantifier elimination]\label{thm:qe}
Let $P = \{P_1,\dots,P_s\} \subset \R[X,Y]$ be a family of $s$ polynomials
of degree at most $d$ in variables $X \in \R^{k}$ and $Y \in \R^{\ell}$.
Let $\Phi(X,Y)$ be a Boolean combination of atomic predicates of the form
$P_i(X,Y) \;\sigma\; 0$, where $\sigma \in \{<,=,>\}$. Consider the existential formula
\[
\psi(Y) := \exists X \in \R^k \;\; \Phi(X,Y).
\]
Then $\psi(Y)$ is equivalent to a quantifier-free formula of the form
\[
\Psi(Y)
=
\bigvee_{i=1}^{I}
\bigwedge_{j=1}^{J_i}
\left(
\bigvee_{n=1}^{N_{ij}} \operatorname{sign}(Q_{ijn}(Y)) = \sigma_{ijn}
\right),
\]
where each $Q_{ijn} \in \R[Y]$ is a polynomial and $\sigma_{ijn} \in \{-1,0,1\}$.
Moreover,
\[
I \le s^{(k+1)(\ell+1)} d^{O(k\ell)}, \quad
J_i \le s^{(k+1)} d^{O(k)}, \quad
N_{ij} \le d^{O(k)}, \quad
\deg(Q_{ijn}) \le d^{O(k)}.
\]
\end{theorem}

\begin{corollary}[One-block quantifier elimination: polynomial form]
\label{cor:qe-poly}
Let $X \in \R^k$ and $Y \in \R^\ell$, and let
\[
\psi(Y) := \exists X \in \R^k\; \Phi(X,Y),
\]
where $\Phi(X,Y)$ is a quantifier-free Boolean formula involving at most
$s$ distinct polynomial sign conditions
\[
P(X,Y)\;\sigma\;0,
\qquad \sigma \in \{<,=,>\},
\]
with each polynomial $P$ having degree at most $d$.

Then $\psi(Y)$ is equivalent to a quantifier-free formula $\Psi(Y)$ whose
truth value is determined by the sign pattern of a family of polynomials
\[
\mathcal P \subseteq \R[Y]
\]
such that:
\begin{itemize}
    \item $|\mathcal P| \le s^{O(k\ell)} d^{O(k\ell)}$,
    \item every $P \in \mathcal P$ has degree at most $d^{O(k)}$.
\end{itemize}
\end{corollary}

\begin{proof}
Apply Theorem~\ref{thm:qe} to the quantifier-free formula $\Phi(X,Y)$,
which involves at most $s$ polynomial sign conditions of degree at most $d$.
This yields an equivalent quantifier-free formula $\Psi(Y)$ whose truth
value is determined by a Boolean combination of sign conditions of
polynomials in $Y$.

The length of $\Psi$ is bounded (up to a constant factor) by the size
of its Boolean representation, which is at most
\[
I \cdot \max_i J_i \cdot \max_{i,j} N_{ij}.
\]
Substituting the bounds from Theorem~\ref{thm:qe} gives
\[
L \le s^{O(k\ell)} d^{O(k\ell)}.
\]
The degree bound follows directly from the theorem.
\end{proof}

\section{Strategic expansion breaks learnability}
\label{app:learnability_breaks}

The aim of this section is to prove Main Theorem~\ref{main:vc-blowup}.

\begin{restatedtheorem}{main:vc-blowup}
For $r > 0$, let $\mathcal{N}_r$ be the neighborhood system given by $x' \in  N_{x} \iff |x - x'| \leq r$.

There exists a hypothesis class $\mathcal{H}$ over $\mathbb R$ such that $\VCdim(\mathcal H) = 1$ but for every $r > 0$, $\VCdim(\mathcal H^{\mathcal N_r}) = \infty$.
\end{restatedtheorem}

We split the proof into two steps. First, in
Lemma~\ref{thm:pathology-1d}, we construct, for every \(n\), a hypothesis class
on an interval whose VC dimension is \(1\), but whose thickening has VC
dimension \(n\). Then, in Corollary~\ref{cor:pathology-1d}, we extend this
construction to obtain Main Theorem~\ref{main:vc-blowup}.

We also prove a related pathology in
Lemma~\ref{lem:pathology-interval-neighborhoods}: there exists a hypothesis class
of VC dimension \(1\) and a neighborhood system whose associated neighborhood
class has VC dimension \(1\), while the induced strategic class is not
PAC learnable.

In this section, we identify binary hypotheses with their positive regions.

\begin{figure}[t]
\centering
\begin{tikzpicture}[scale=1.05]

\def\r{0.4}
\def\R{0.7}

\draw[->] (-2.5,0) -- (10,0) node[right] {$\mathbb R$};

\foreach[count=\k] \i in {-1.5,0.5, 2.5, 4.5, 6.5 ,8.5} {
    \draw[dashed, gray!55] ({\i-\R},-0.35) rectangle ({\i+\R},0.35);
    \node[gray!70] at (\i,0.55) {$U_{\k}$};
}

\foreach[count=\k] \i in {-1.5,0.5, 2.5, 4.5, 6.5 ,8.5} {
    \draw[red!35, line width=7pt] ({\i-\r},0) -- ({\i+\r},0);
    \fill[red!80!black] (\i,0) circle (2.2pt);
    \node[below=9pt] at (\i,0) {$p_{\k}$};
}

\fill[blue] (-1.75,0) circle (2.2pt);
\fill[blue] (0.75,0) circle (2.2pt);
\fill[blue] (4.2,0) circle (2.2pt);
\fill[blue] (8.7,0) circle (2.2pt);

\fill[gray!70] (1.95, 0) circle (2.2pt);
\fill[gray!70] (7.05,0) circle (2.2pt);

\node[blue!80!black] at (3.5,1.5)
{$q_i^S$ inside $N_{p_i}$ iff $i\in S$};

\node[gray!70!black] at (3.5,1)
{$q_i^S$ outside $N_{p_i}$ iff $i\notin S$};

\node[red!70!black] at (3.5,-1)
{red intervals are the neighborhoods $N_{p_i}$};

\node at (3.5,-1.5)
{$S=\{1,3,4,6\}$, hence after the strategic transformation only $p_1,p_3,p_4,p_6$ are captured};

\end{tikzpicture}
\caption{
Each anchor $p_i$ has its own private cell $U_i$. 
For the set $F_S=\{q_i^S:i\in[n]\}$, the point $q_i^S$ is placed inside
\(N_{p_i}\) exactly when \(i\in S\), and outside \(N_{p_i}\) otherwise.
Thus \((\mathbf 1_{F_S})^{\mathcal N_s}\cap\{p_1,\dots,p_n\}=\{p_i:i\in S\}\).
}
\label{fig:blowup-line}
\end{figure}
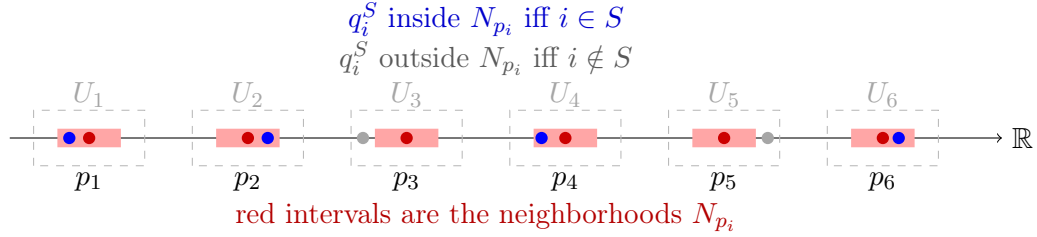

\begin{lemma}[VC blows up after interval neighborhoods]
\label{thm:pathology-1d}

For every \(n\ge 1\) and every \(r \ge r' > 0\), there exist a finite hypothesis
class \(\mathcal H\subseteq\{0,1\}^{\R}\) with \(|\mathcal H|=2^n\) and a neighborhood
system
\[
\mathcal N_s=\{N_x:x\in\R\},
\qquad
y\in N_x \iff |x-y|\le s,
\]
such that for every \(s\in[r',r]\),
\[
\VCdim(\mathcal H)=1
\qquad\text{but}\qquad
\VCdim(\mathcal H^{\mathcal N_s})=n.
\]
\end{lemma}

\begin{proof}
Fix \(r>0\).

\textbf{Well-separated anchor points.}
Let \(p_1,\dots,p_n\in\R\) be such that
\[
|p_i-p_j|>4r \quad (i\neq j),
\]
for example \(p_i=10ri\).

For each \(i\), define a small interval
\[
U_i := (p_i-2r,\,p_i+2r),
\]
so the intervals \(U_i\) are pairwise disjoint.

\medskip
\textbf{Encoding subsets.}
For each subset \(S\subseteq[n]\), define a set \(F_S\) by choosing one point
\(q_i^S\in U_i\) for every \(i\in[n]\) such that
\[
|q_i^S-p_i|<r' \quad\text{if } i\in S,
\qquad
|q_i^S-p_i|>r \quad\text{if } i\notin S.
\]
Set
\[
F_S := \{q_1^S,\dots,q_n^S\}.
\]

Since each \(U_i\) contains infinitely many points, we may choose all
\(q_i^S\) so that the sets \(F_S\) are pairwise disjoint.

\medskip
\textbf{VC dimension of \(\mathcal H\).}
Let
\[
\mathcal H:=\{\mathbf 1_{F_S}:S\subseteq[n]\}.
\]

The sets \(F_S\) are nonempty and pairwise disjoint, hence \(\mathcal H\)
shatters a singleton but cannot shatter any two-point set. Therefore
\[
\VCdim(\mathcal H)=1.
\]

\medskip
\textbf{Effect of the transformation.}
We claim that \(\mathcal H^{\mathcal N_s}\) shatters
\(\{p_1,\dots,p_n\}\) whenever \(r' \le s \le r\).

Fix \(i\in[n]\):
\begin{itemize}
\item If \(i\in S\), then \(|q_i^S-p_i|<r' \le s\), hence
\[
q_i^S\in N_{p_i},
\]
and therefore
\[
p_i\in (\mathbf 1_{F_S})^{\mathcal N_s}.
\]

\item If \(i\notin S\), then
\[
|q_i^S-p_i|>r \ge s,
\quad\text{and for } j\neq i,\quad
|q_j^S-p_i|>r \ge s
\]
(by separation of the \(U_j\)). Hence
\[
F_S\cap N_{p_i}=\emptyset,
\]
so
\[
p_i\notin (\mathbf 1_{F_S})^{\mathcal N_s}.
\]
\end{itemize}

Thus
\[
(\mathbf 1_{F_S})^{\mathcal N_s}\cap\{p_1,\dots,p_n\}
=
\{p_i:i\in S\},
\]
so all subsets are realized. Hence \(\mathcal H^{\mathcal N_s}\)
shatters an \(n\)-point set, and therefore
\[
\VCdim(\mathcal H^{\mathcal N_s})=n.
\]
\end{proof}

\begin{corollary}[A single class that blows up for every radius]
\label{cor:pathology-1d}
There exists an infinite hypothesis class \(\mathcal H\subseteq\{0,1\}^{\R}\) such that
\[
\VCdim(\mathcal H)=1,
\qquad\text{but for every } s>0,\qquad
\VCdim(\mathcal H^{\mathcal N_s})=\infty,
\]
where \(\mathcal N_s=\{N_x:x\in\R\}\) is given by
\[
y\in N_x \iff |x-y|\le s.
\]
\end{corollary}

\begin{proof}
For \(n\ge1\) and \(m\in\mathbb Z\), let \(\mathcal H_{n,m}\) be the hypothesis
class given by Lemma~\ref{thm:pathology-1d} with
\[
r=2^{-m},\qquad r'=2^{-m-1}.
\]
Equivalently, write
\[
\mathcal H_{n,m}=\{\mathbf 1_{F}:F\in\mathcal A_{n,m}\},
\]
where \(\mathcal A_{n,m}\) is the corresponding family of subsets of \(\R\).
From the construction, we may assume that all sets in \(\mathcal A_{n,m}\)
are contained in an interval of length
\[
L_{n,m} := 10\cdot 2^{-m}(n+1).
\]

Enumerate the pairs \((n,m)\in\mathbb N\times\mathbb Z\) as
\[
(n_1,m_1),(n_2,m_2),\ldots .
\]
Define shifts inductively by
\[
a_1:=0,
\qquad
a_{t+1}:=a_t + L_{n_t,m_t} + 1,
\]
and set
\[
\mathcal H :=
\bigcup_{t=1}^\infty
\{\mathbf 1_{F+a_t}:F\in \mathcal A_{n_t,m_t}\}.
\]

By construction, the supports of the shifted families are pairwise disjoint.
Since within each block the sets are nonempty and pairwise disjoint, it follows
that all supports of hypotheses in \(\mathcal H\) are nonempty and pairwise
disjoint. Hence
\[
\VCdim(\mathcal H)=1.
\]

Now fix \(s>0\), and choose \(m\in\mathbb Z\) such that
\[
2^{-m-1} \le s \le 2^{-m}.
\]
For every \(n\ge1\), there exists some block corresponding to the pair \((n,m)\).
By Lemma~\ref{thm:pathology-1d}, the strategic transformation of that block by
\(\mathcal N_s\) shatters an \(n\)-point set. Therefore
\[
\VCdim(\mathcal H^{\mathcal N_s}) \ge n
\]
for all \(n\), which implies
\[
\VCdim(\mathcal H^{\mathcal N_s})=\infty.
\]
\end{proof}

\begin{lemma}[Pathology with interval neighborhoods]
\label{lem:pathology-interval-neighborhoods}
There exist a hypothesis class \(\mathcal H\subseteq \{0,1\}^{\mathbb R}\) and a
neighborhood system \(\mathcal N=\{N_x:x\in\mathbb R\}\) such that
\[
\VCdim(\mathcal H)=1,
\qquad
\VCdim(\{N_x:x\in\mathbb R\})=1,
\]
but
\[
\VCdim(\mathcal H^{\mathcal N})=\infty.
\]
\end{lemma}

\begin{proof}
Partition \(\mathbb R\) into disjoint intervals
\[
I_j:=[j,j+1), \qquad j\in\mathbb Z,
\]
and define
\[
N_x:=I_j \quad \text{whenever } x\in I_j.
\]
In other words, $N_x = [\lfloor x \rfloor, \lfloor x \rfloor + 1)$.
Thus every point has a neighborhood which is an interval containing it,
and the family \(\{N_x:x\in\mathbb R\}\) consists of pairwise disjoint
sets, hence
\[
\VCdim(\{N_x:x\in\mathbb R\})=1.
\]

Now fix points
\[
p_i:=i+\tfrac12 \in I_i, \qquad i\ge 1.
\]
Choose an injective map
\[
\alpha\colon 2^{\mathbb N}\to (0,1/10).
\]
For example, identifying $2^{\mathbb N}$ with $\{0,1\}^{\mathbb N}$, we can take $\alpha(x_1x_2\ldots) = 0.01x_1x_2\ldots$, where the right-hand side is interpreted as a decimal expansion.

For each subset \(S\subseteq\mathbb N\), define
\[
F_S
:=
\{\, i+\tfrac12 + \alpha(S) : i\in S \,\}
\;\cup\;
\{\, -i-\alpha(S) : i\notin S \,\}.
\]
Let
\[
h_S:=\mathbf 1_{F_S},
\qquad
\mathcal H:=\{h_S:S\subseteq\mathbb N\}.
\]

\medskip
\textbf{VC dimension of \(\mathcal H\).}
The sets \(F_S\) are nonempty and pairwise disjoint: different \(S\neq T\)
have different shifts \(\alpha(S)\neq\alpha(T)\), so no point can belong to
both \(F_S\) and \(F_T\). Hence the supports of the hypotheses in
\(\mathcal H\) are pairwise disjoint. Therefore \(\mathcal H\) shatters a
singleton but no two-point set, and thus
\[
\VCdim(\mathcal H)=1.
\]

\medskip
\textbf{VC dimension of \(\{N_x : x \in \mathbb R\}\).}
By definition, \(\{N_x : x \in \mathbb R\} = \{I_j : j \in \mathbb Z\}\). Since the sets $I_j$ are disjoint,
\[
\VCdim(\{N_x : x \in \mathbb R\}) = 1.
\]

\medskip
\textbf{Effect of the transformation.}
For \(i\ge1\),
\[
h_S^{\mathcal N}(p_i)=1
\iff
F_S\cap N_{p_i}\neq\emptyset.
\]
Since \(N_{p_i}=I_i=[i,i+1)\), the only points of \(F_S\) that can lie in
this interval are of the form
\[
i+\tfrac12+\alpha(S),
\]
and such a point belongs to \(F_S\) if and only if \(i\in S\). Thus
\[
h_S^{\mathcal N}(p_i)=1
\iff
i\in S.
\]

Therefore
\[
\{\,p_i:h_S^{\mathcal N}(p_i)=1,\ 1\le i\le n\,\}
=
\{p_i:i\in S\cap[n]\}.
\]
Since \(S\subseteq\mathbb N\) was arbitrary, the class \(\mathcal H^{\mathcal N}\)
shatters \(\{p_1,\dots,p_n\}\) for every \(n\), and hence
\[
\VCdim(\mathcal H^{\mathcal N})=\infty. \qedhere
\]
\end{proof}

\section{The role of definability}
\label{app:sec:role-of-definability}

The examples in the previous section are constructive, but they may look
somewhat artificial: in each case, the hypotheses are chosen in a very specific
way so that, after thickening by a neighborhood relation, the transformed class
shatters arbitrarily large finite sets.

In light of Main Theorems~\ref{main:bounds_R} and~\ref{main:bounds_R_exp}, one
might suspect that such behavior is ruled out whenever the original class and
the neighborhood relation are definable in a reasonably tame expansion
of the real field. The following example shows that this is not true in general.
Already in the expansion
\[
  (\R,+,\cdot,>,=,\{\cdot\}),
\]
where \(\{x\}:=x-\lfloor x\rfloor\) denotes the fractional part, a  definable class of VC dimension \(1\), together with a definable neighborhood system of VC dimension \(1\), can give rise to a strategic class of infinite VC dimension.

\begin{lemma}\label{lem:frac}
Fix \(0<r<1/2\). Consider the definable family \(\mathcal H\) given by
\[
\Phi_{\mathcal H}(x,a)
:=
\exists y\,
\Bigl(
  \{y\}=0
  \wedge \{a\}=0
  \wedge y\neq 0
  \wedge x=y+\{\sqrt2\,ay\}
\Bigr),
\]
and let \(\mathcal N_r\) be the neighborhood relation defined by
\[
\Phi_{\mathcal N_r}(x,y):=|y-x|\le r.
\]
Then the strategically transformed family \(\mathcal H^{\mathcal N_r}\) has infinite VC
dimension. The same construction works with \(\sqrt2\) replaced by any fixed
irrational number.
\end{lemma}

\begin{proof}
For each parameter \(a\), let $H_a:=\{x\in\R:\Phi_{\mathcal H}(x,a)\}$ denote the positive region of the corresponding hypothesis.
Write \(\alpha=\sqrt2\). The idea is simple: the set \(H_a\) contains one point
in each integer cell, namely \(y+\{\alpha ay\}\). By choosing the integer
parameter \(a\), we can force these fractional parts to fall either close to
\(1/2\), so that the midpoint \(y+1/2\) is hit after \(r\)-thickening, or far
from \(1/2\), so that it is not hit.

For integer \(a\), the set defined by \(\Phi_{\mathcal H}(x,a)\) is
\[
H_a=\{\,y+\{\alpha ay\}:y\in\Z\setminus\{0\}\,\},
\]
and for non-integer \(a\) it is empty.

First note that the nonempty sets \(H_a\) are pairwise disjoint. Indeed, if
\(x\in H_{a_1}\cap H_{a_2}\), then
\[
x=y_1+\{\alpha a_1y_1\}
 =
y_2+\{\alpha a_2y_2\}
\]
for some \(y_1,y_2\in\Z\setminus\{0\}\). Since the two sides lie in the
half-open intervals \([y_1,y_1+1)\) and \([y_2,y_2+1)\), we have
\(y_1=y_2=:y\). Hence
\[
\{\alpha a_1y\}=\{\alpha a_2y\},
\]
so \(\alpha y(a_1-a_2)\in\Z\). Since \(\alpha\) is irrational and
\(y\neq 0\), this implies \(a_1=a_2\). Thus the original family has
\(\VCdim(\mathcal H)=1\).

We now show that the thickened family has infinite VC dimension. Let
\[
P:=(1/2-r,1/2+r),
\qquad
Q:=(0,1/2-r).
\]
Both intervals are nonempty. We claim that for every \(n\) there are positive
integers \(b_1,\ldots,b_n\) such that for every \(A\subseteq[n]\) there is a
nonempty interval \(I_A\subseteq[0,1)\) satisfying
\[
\{b_i t\}\in P \quad\text{if } i\in A,
\qquad
\{b_i t\}\in Q \quad\text{if } i\notin A
\]
for all \(t\in I_A\).

This is an elementary induction. Suppose \(b_1,\ldots,b_k\) and the intervals
\(I_A\), \(A\subseteq[k]\), have already been constructed, and put
\[
v:=\min_{A\subseteq[k]} |I_A|>0.
\]
Choose \(b_{k+1}\) so large that \(b_{k+1}>2/v\). Then each interval \(I_A\)
contains a full interval of the form
\[
[j/b_{k+1},(j+1)/b_{k+1}).
\]
Inside such a full interval, the map \(t\mapsto \{b_{k+1}t\}\) runs once through
\([0,1)\). Therefore both intersections
\[
I_A\cap\{t:\{b_{k+1}t\}\in P\},
\qquad
I_A\cap\{t:\{b_{k+1}t\}\in Q\}
\]
contain nonempty intervals. Use the first one to define
\(I_{A\cup\{k+1\}}\), and the second one to redefine \(I_A\). This completes
the induction.

Now fix \(A\subseteq[n]\). Since \(\alpha\) is irrational, the sequence
\(\{\alpha m\}\), \(m\in\Z\), is dense in \([0,1)\). Choose
\(m_A\in\Z\setminus\{0\}\) such that $\{\alpha m_A\}\in I_A$. Then, since each \(b_i\) is an integer,
\(
\{\alpha m_A b_i\}
=
\{b_i\{\alpha m_A\}\}.
\)
Hence
\[
\{\alpha m_A b_i\}\in P \quad\text{if } i\in A,
\qquad
\{\alpha m_A b_i\}\in Q \quad\text{if } i\notin A.
\]

We claim that the points
\[
x_i:=b_i+\frac12,
\qquad i=1,\ldots,n,
\]
are shattered. For the parameter \(m_A\), the corresponding point of
\(H_{m_A}\) in the \(b_i\)-th integer cell is
\[
z_i:=b_i+\{\alpha m_A b_i\}.
\]
If \(i\in A\), then \(\{\alpha m_A b_i\}\in P\), so
\(
|x_i-z_i|<r,
\)
and therefore \(x_i\in H_{m_A}^{\mathcal N_r}\).

If \(i\notin A\), then \(\{\alpha m_A b_i\}\in Q\), so the point \(z_i\) in the
same integer cell is at distance \(>r\) from \(x_i\). Points of \(H_{m_A}\)
coming from other integer cells are at distance \(>1/2>r\) from \(x_i\). Hence
\(
x_i\notin H_{m_A}^{\mathcal N_r}.
\)
Thus
\[
x_i\in H_{m_A}^{\mathcal N_r}
\quad\Longleftrightarrow\quad
i\in A.
\]
Since every \(A\subseteq[n]\) is realized, the set
\(\{x_1,\ldots,x_n\}\) is shattered. As \(n\) was arbitrary,
\[
\VCdim(\mathcal H^{\mathcal N_r})=+\infty.
\]

The proof used only the irrationality of \(\sqrt2\), so the same argument works
with any fixed irrational number.
\end{proof}

\begin{remark}[Defining $\mathbb{Z}$ is essentially sufficient]
\label{rem:definable-Z-suffices}
The construction above relies only on two ingredients:
\begin{enumerate}
    \item the ability to restrict variables to integers, and
    \item the presence of a fixed irrational constant $\alpha$.
\end{enumerate}


Whenever the set \(\mathbb Z\) is definable, the fractional-part function
\(\{x\}\) is also definable in the structure, since
\[
\{x\}=x'
\iff
\exists y\,(\theta(y)\wedge x'=x-y\wedge 0\le x'<1),
\]
and the proof goes through unchanged.

In particular, the role of the predicates $\{y\}=0$ and $\{a\}=0$ in the
definition of $\Phi_{\mathcal H}$ is only to enforce $y,a\in\mathbb{Z}$.

Thus, the phenomenon of VC-dimension blow-up after thickening already appears
in any expansion of the real field in which $\mathbb{Z}$ is definable.
\end{remark}

\section{Proof of Main Theorem~\ref{main:R_exp}}
\label{app:sec:proof-Rexp-without-bounds}

\begin{restatedtheorem}{main:R_exp}
Let $\mathcal H$ be a hypothesis class over~$\mathbb R^l$ definable in $\mathbb{R}_{\exp}$ with $k$ parameters, and let $\mathcal N$ be a neighborhood system definable in~$\mathbb{R}_{\exp}$. Then the strategic class $\mathcal H^\mathcal N$ is PAC learnable.

Moreover, the ERM sample complexity of $\mathcal H^\mathcal N$ satisfies
\[
m_{\mathrm{ERM}}^{\mathrm{real}}(\varepsilon,\delta)
=
O\Bigl(
\frac{k\log(1/\varepsilon) + \log(1/\delta) + K_{\mathcal{H},\mathcal{N}}}
{\varepsilon}
\Bigr),
\]
where $K_{\mathcal H,\mathcal N}$ is a constant depending on $\mathcal H$ and $\mathcal N$.
\end{restatedtheorem}

\begin{proof}
The strategic class \(\mathcal H^\mathcal N\) is again definable in
\(\mathbb R_{\exp}\). Indeed, suppose \(\mathcal H\) is defined by
\(\Phi_{\mathcal H}(y,a)\), with parameter \(a\in\mathbb R^k\), so that
\[
h_a
=
\mathbf 1{\!\left\{y\in\R^l:\Phi_{\mathcal H}(y,a)\right\}}.
\]
Then
\[
h_a^{\mathcal N}(x)=1
\iff
\exists y\bigl(\Phi_{\mathcal N}(x,y)\wedge \Phi_{\mathcal H}(y,a)\bigr).
\]
Thus \(\mathcal H^\mathcal N\) is defined by the single formula
\[
\psi(x,a)
:=
\exists y\bigl(\Phi_{\mathcal N}(x,y)\wedge \Phi_{\mathcal H}(y,a)\bigr),
\]
with the same parameter vector \(a\in\mathbb R^k\).

Since \(\mathbb R_{\exp}\) is o-minimal \cite{Wilkie1996modelcompleteness}, the polynomial growth theorem for definable families in o-minimal structures \cite[Corollary 4.6]{LaskowskiCompression} implies that there exists a constant \(C_\psi\ge 1\), depending only on the defining formula \(\psi\), such that
\[
\Pi_{\mathcal H^\mathcal N}(m)\le C_\psi m^k
\]
for every \(m\ge 1\). Write
\[
K:=\log C_\psi .
\]
Applying Lemma~\ref{lem:erm-from-growth} to \(\mathcal H^\mathcal N\),
we obtain
\[
m_{\mathrm{ERM}}^{\mathrm{real}}(\varepsilon,\delta)
=
O\left(
\frac{
k\log(k/\varepsilon)+\log C_\psi+\log(1/\delta)
}{\varepsilon}
\right).
\]
Since \(k\) is fixed for the class, the term \(k\log k\) is absorbed into
the class-dependent constant. Hence
\[
m_{\mathrm{ERM}}^{\mathrm{real}}(\varepsilon,\delta)
=
O\left(
\frac{
k\log(1/\varepsilon)+K+\log(1/\delta)
}{\varepsilon}
\right).
\]

Finally, since the growth function of \(\mathcal H^\mathcal N\) is polynomially
bounded, \(\mathcal H^\mathcal N\) has finite VC dimension. Therefore
\(\mathcal H^\mathcal N\) is PAC learnable.
\end{proof}

\section{Proof of Theorem~\ref{main:bounds_R}}
\label{app:sec:proof-R-bounds}

\begin{proposition}[Growth bound for definable strategic classes over $\R$]
\label{prop:growth-R}
Let $\mathcal H=\{h_a:a\in\R^k\}$ be a hypothesis class on $\R^l$, and let
$\mathcal N$ be a neighborhood system on $\R^l$. Suppose that
\(h_a(x)=1\) and the predicate \(\Phi_{\mathcal N}(x,y)\) defining
\(y\in N_x\) are definable in the real field $\R$ by quantifier-free
formulas of \emph{total description complexity} at most $D$
(that is, the formulas involve at most $s$ polynomials, each of degree at
most~$D'$, with $D=sD'$).

Then, for every $m\ge 1$,
\[
\Pi_{\mathcal H^{\mathcal N}}(m)
\le
(Cm)^k D^{O(k^2l)}.
\]
\end{proposition}

\begin{proof}
Fix points $y_1,\dots,y_m\in\R^l$. For $a\in\R^k$, define
\[
\varphi(y,a)
\iff
\exists x\in\R^l
\bigl(h_a(x)=1 \wedge \Phi_{\mathcal N}(y,x)\bigr).
\]
Then
\[
h_a^{\mathcal N}(y)=1
\iff
\varphi(y,a).
\]

For each fixed $i\in[m]$, the condition $h_a^{\mathcal N}(y_i)=1$ is
equivalent to
\[
\exists x\in\R^l
\bigl(h_a(x)=1 \wedge \Phi_{\mathcal N}(y_i,x)\bigr).
\]
This is an existential formula with quantified variables $x\in\R^l$ and
free variables $a\in\R^k$. By assumption, its quantifier-free part involves
at most $s$ polynomial sign conditions of degree at most $D'$.

By Corollary~\ref{cor:qe-poly}, for each fixed $y_i$ this formula
is equivalent to a quantifier-free formula in $a$, whose truth value is
determined by the sign pattern of at most
\[
L\le s^{O(lk)}D'^{O(lk)} = D^{O(lk)}
\]
polynomials in $a$, each of degree at most
\[
\Delta\le D'^{O(l)} \leq D^{O(l)}.
\]
Thus, for each $i$, there is a family
\[
\mathcal P_i\subseteq \R[a],
\qquad
|\mathcal P_i|\le D^{O(lk)},
\]
such that the truth value of $h_a^{\mathcal N}(y_i)=1$ is determined by
the sign pattern of $\mathcal P_i$ at $a$.

Let
\[
\mathcal P=\bigcup_{i=1}^m \mathcal P_i.
\]
Then
\[
M:=|\mathcal P|
\le
m\,D^{O(lk)},
\]
and every polynomial in $\mathcal P$ has degree at most
\[
\Delta\le D^{O(l)}.
\]

For every $a\in\R^k$, the labeling of
$\{y_1,\dots,y_m\}$ by $h_a^{\mathcal N}$ is determined by the sign pattern
of the polynomials in $\mathcal P$ at $a$. Hence the number of traces is at
most the number of sign patterns of $M$ polynomials of degree at most
$\Delta$ in $k$ variables.

By the standard bound on sign patterns of real polynomials
(see, e.g., Warren~\cite{warren1968} or
\cite[Proposition~5.5]{alon1995tools}), the number of such sign patterns is at most
\[
\left(\frac{C\, M \Delta}{k}\right)^k.
\]
Substituting the bounds on $M$ and $\Delta$ gives
\[
\Pi_{\mathcal H^{\mathcal N}}(m)
\le
\left(C\,D^{O(l)}\,m\,D^{O(lk)}\right)^k.
\]
Therefore
\[
\Pi_{\mathcal H^{\mathcal N}}(m)
\le
(Cm)^k D^{O(k^2l)}.
\]
Since the sample was arbitrary, the proof is complete.
\end{proof}

\begin{remark}[Relation to Goldberg--Jerrum]
The proof above follows the proof of Theorem~\ref{thm:GJ-param} of Goldberg and Jerrum:
once membership is determined by the sign pattern of finitely many
polynomials in the parameter variables, one obtains bounds on the growth
function and VC dimension via bounds on the number of realizable sign
patterns.

In the quantifier-free setting, if membership $x\in F_a$ is determined by
the sign pattern of at most $s$ polynomials in $a\in\R^k$, each of degree
at most $D$, then the Goldberg--Jerrum bound yields
\[
\VCdim(\mathcal F)=O\bigl(k\log(sD)\bigr).
\]

In our setting, the strategic transformation introduces an existential
quantifier. Applying quantitative one-block quantifier elimination reduces
the problem to the quantifier-free regime, but increases the number of
defining polynomials: the condition $y\in C_a^{\mathcal N}$ becomes
determined by the sign pattern of at most
\[
s^{O(kl)}D^{O(kl)}
\]
polynomials in the parameter variables $a$.

\end{remark}

\begin{restatedtheorem}{main:bounds_R}
Let~$\mathcal{H}$ be a hypothesis class over $\mathbb R^l$ definable with $k$ parameters in the language of $\mathbb{R}$, and let $\mathcal N$ be a neighborhood system definable in the same language. Suppose that both $\mathcal H$ and $\mathcal N$ are defined by quantifier-free formulas involving polynomial equalities and inequalities whose \emph{total description complexity} is at most $D$ (that is, the formulas involve at most $s$ polynomials, each of degree at most~$D'$, with $D = sD'$). 
Then
\[
  m_{\mathrm{ERM}}^{\mathrm{real}}(\varepsilon,\delta)
  =
  O\Bigl(
    \frac{
      k\log(1/\varepsilon)
      + k^2 l \log D
      + \log(1/\delta)
    }{\varepsilon}
  \Bigr),
\]
and
\[
  \VCdim(\mathcal H^\mathcal N)
  =
  O\bigl(k^2 l \log D\bigr).
\]
\end{restatedtheorem}

\begin{proof}
By Proposition~\ref{prop:growth-R}, for every \(m\ge 1\),
\[
\Pi_{\mathcal H^{\mathcal N}}(m)
\le
 (C_0m)^k D^{O(k^2l)}.
\]
Equivalently, for some constant \(A>0\),
\[
\Pi_{\mathcal H^{\mathcal N}}(m)\le C m^k,
\qquad
C:=C_0^k D^{A k^2l}.
\]
Hence
\[
\log C=O(k^2l\log D).
\]

Applying Lemma~\ref{lem:erm-from-growth} to \(\mathcal H^{\mathcal N}\), we obtain
\[
m_{\mathrm{ERM}}^{\mathrm{real}}(\varepsilon,\delta)
=
O\left(
\frac{
k\log(k/\varepsilon)
+
k^2l\log D
+
\log(1/\delta)
}{\varepsilon}
\right).
\]
Absorbing the lower-order \(k\log k\) term into the displayed bound gives
\[
m_{\mathrm{ERM}}^{\mathrm{real}}(\varepsilon,\delta)
=
O\left(
\frac{
k\log(1/\varepsilon)
+
k^2l\log D
+
\log(1/\delta)
}{\varepsilon}
\right).
\]

For the VC-dimension bound, applying Lemma~\ref{lem:vc-from-growth} to
\(\mathcal H^{\mathcal N}\) gives
\[
\VCdim(\mathcal H^{\mathcal N})
\le
2\log C+4k\log(4k).
\]
Since \(\log C=O(k^2l\log D)\), this yields
\[
\VCdim(\mathcal H^{\mathcal N})
=
O(k^2l\log D).
\]
\end{proof}

\section{Proof of Main Theorem~\ref{main:bounds_R_exp}}
\label{app:sec:proof-main:bounds_R_exp}

The aim of this section is to prove the following theorem.

\begin{restatedtheorem}{main:bounds_R_exp}
Let~$\mathcal H$ be a hypothesis class over $\mathbb R^l$ definable with $k$ parameters in the language of $\mathbb R_{\exp}$, and let $\mathcal N$ be a neighborhood system definable in the language of $\mathbb R_{\exp}$. Suppose, moreover, that both $\mathcal H$ and $\mathcal N$ are defined using existential formulas of format $F$ and degree $D$. Then
\[
  m_{\mathrm{ERM}}^{\mathrm{real}}(\varepsilon,\delta)=  O\Bigl(\frac{k\log(1/\varepsilon) + \gamma(F) \log(D) + \log(1/\delta)}{\varepsilon}\Bigr)
\]
and
\[
\VCdim(\mathcal H^\mathcal N)  = O_F(\log D),
\]
where $\gamma\colon \mathbb N \to \mathbb N$ is an explicitly computable function.
\end{restatedtheorem}

The proof of Main Theorem~\ref{main:bounds_R_exp} uses several notions from the
theory of restricted Pfaffian and sub-Pfaffian sets. Since these notions are
somewhat abstract and not directly part of the statement of the theorem, we
first describe the structure of the appendix.

\paragraph{Organization.} In \Cref{app:restricted-definition}, we recall the restricted Pfaffian
structure \(\R_{\rPfaff}\) and the basic notions of Pfaffian functions,
semi-Pfaffian sets, sub-Pfaffian sets, format, and degree. In
\Cref{app:star-format-degree}, we explain the refined notions of
\(*\)-format and \(*\)-degree used in the quantitative cell decomposition
results. In \Cref{app:rPfaff-growth-function}, we use sharp cell decomposition
in \(\R_{\rPfaff}\) to prove a growth-function bound for definable hypothesis
classes in this structure. In \Cref{app:embedding-exp-rPfaff}, we explain how
formulas in \(\R_{\exp|_{[-M,M]}}\) can be interpreted in
\(\R_{\rPfaff}\), with controlled complexity. Finally, in
\Cref{app:proof-main-three}, we combine these ingredients to prove
Main Theorem~\ref{main:bounds_R_exp}.

\subsection{Restricted Pfaffian structure}
\label{app:restricted-definition}

\begin{definition}[Pfaffian function]\label{def:pfaff}
Let \(U \subseteq \R^n\) be open. A \emph{Pfaffian chain} on \(U\) is a sequence
of analytic functions
\[
  f_1,\ldots,f_r\colon U \to \R
\]
such that, for every \(1 \le i \le r\) and \(1 \le j \le n\), there is a
polynomial
\[
  P_{ij} \in \R[x_1,\ldots,x_n,y_1,\ldots,y_i]
\]
satisfying
\[
  \frac{\partial f_i}{\partial x_j}(x)
  =
  P_{ij}\bigl(x,f_1(x),\ldots,f_i(x)\bigr).
\]

A function \(f\colon U\to \R\) is called a \emph{Pfaffian function (with respect to this chain)}
if there exists a polynomial \(Q\) such that
\[
  f(x)=Q\bigl(x,f_1(x),\ldots,f_r(x)\bigr).
\]

We say that the Pfaffian function \(f\) has \emph{format} \(n+r\) and
\emph{degree} \(\sum_{i,j} \deg P_{ij} + \deg Q\).

In analogy with semialgebraic sets, a set \(X \subseteq U\) is called
\emph{semi-Pfaffian} if it can be defined by a Boolean combination of
atomic formulas of the form
\[
  g(x)\,\sigma\,0,
  \qquad
  \sigma\in\{=,>,<,\ge,\le\},
\]
where each \(g\colon U\to\R\) is a Pfaffian function.

Suppose that \(X\) is defined using Pfaffian functions
\(g_1,\ldots,g_m\), all taken with respect to a common Pfaffian chain
\(f_1,\ldots,f_r\). Then we say that this representation has
\emph{format} \(n+r\), and its \emph{degree} is the sum of the degrees
of the functions \(g_i\) appearing in the atomic formulas.

If the functions \(g_1,\ldots,g_m\) are originally given with respect to
different Pfaffian chains, we may pass to a single Pfaffian chain by
concatenation. In this way, the notion of format and degree is well-defined
up to an additive increase in the format, while the degree remains bounded
by the total degree of the defining data.
\end{definition}

A \emph{restricted Pfaffian function} on the unit cube \(I^n\), where
\(I=[0,1]\), is the restriction to \(I^n\) of a Pfaffian function defined on
an open neighborhood of \(I^n\).
We denote by \(\R_{\rPfaff}\) the
expansion of the real field by all such restricted Pfaffian functions, in all
arities, viewed as total functions on \(\R^n\) by extending them by zero
outside \(I^n\).

Throughout this section, we call a set \emph{sub-Pfaffian} if it is definable
in \(\R_{\rPfaff}\). By Wilkie's theorem of the complement,
\(\R_{\rPfaff}\) is o-minimal~\cite{Wilkie1999Complement}. Moreover, every
sub-Pfaffian set \(X \subseteq \R^n\) admits a representation of the form
\[
X = \pi(Y),
\]
where \(Y \subseteq \R^{n+d}\) is a semi-Pfaffian set and
\(\pi\colon \R^{n+d} \to \R^n\) is the projection onto the first \(n\)
coordinates. The \emph{format} and \emph{degree} of this representation are
defined to be the format and degree of \(Y\).

Equivalently, every formula \(\Phi(x)\) in \(\R_{\rPfaff}\) is equivalent to
one of the form
\[
\Phi(x) \;\equiv\; \exists y \in \R^d \;\theta(x,y),
\]
where \(\theta(x,y)\) is quantifier-free and is a Boolean combination of
atomic predicates
\[
g(x,y)\,\sigma\,0,
\qquad
\sigma\in\{=,>,<,\ge,\le\},
\]
with \(g\) restricted Pfaffian functions. Thus, every definable set is
\(\Sigma_1\)-definable.

In contrast, in the real field \(\R\), every formula is equivalent to a
quantifier-free formula (quantifier elimination).

\begin{remark}
When referring to the format and degree of a sub-Pfaffian set, we implicitly
fix a particular representation as a projection of a semi-Pfaffian set.
A given set may admit multiple such representations with different
complexities. Since our results only require upper bounds, this choice is
immaterial: when we say that a sub-Pfaffian set has format at most \(F\)
and degree at most \(D\), we mean that it admits some representation with
these bounds.
\end{remark}

\subsection{Sharp cell decomposition}
\label{app:star-format-degree}

In this section we record a quantitative cell decomposition result for
\(\R_{\rPfaff}\), which we will use as a black box.

\paragraph{Cylindrical cell decomposition.}
We use the notion of cylindrical cell and cylindrical cell decomposition
from Definition~\ref{def:cell-decomposition}.

We refer to the standard literature for further details; see, e.g.,
\cite[Chapter~5]{basu2006} in the semialgebraic setting and
\cite[Definitions~5--6]{binyamini2020effective} for the Pfaffian setting.

\paragraph{Example (Cylindrical decomposition of the unit disk).}
Consider the unit disk
\[
X = \{(x,y)\in \R^2 : x^2 + y^2 \le 1\}.
\]

A cylindrical cell decomposition of $\R^2$ compatible with $X$ is obtained as follows. 

\medskip
\noindent
\textbf{Decomposition of the $x$-axis.}

We first decompose $\R$ into the cells
\[
S_1 = (-\infty,-1),\quad
S_2 = \{-1\},\quad
S_3 = (-1,1),\quad
S_4 = \{1\},\quad
S_5 = (1,\infty).
\]

\medskip
\noindent
\textbf{Lifting over each base cell.}

\smallskip
\emph{(i) Over $S_1$ and $S_5$.}

If $x \in S_1 \cup S_5$, then $x^2 > 1$, so the inequality
\[
x^2 + y^2 \le 1
\]
has no solutions in $y$. Thus, no defining functions arise, and we obtain a single cylindrical cell:
\[
S_i \times \R, \quad i=1,5.
\]

\smallskip
\emph{(ii) Over $S_2$ and $S_4$.}

At $x=\pm 1$, the equation becomes
\[
y^2 \le 0,
\]
so the only solution is $y=0$. This corresponds to a single defining function (constant zero), and we obtain three cells:
\[
S_i \times (-\infty,0),\quad
S_i \times \{0\},\quad
S_i \times (0,\infty),
\quad i=2,4.
\]

\smallskip
\emph{(iii) Over $S_3 = (-1,1)$.}

For $x \in (-1,1)$, the boundary of $X$ is given by the two functions
\[
\xi_1(x) = -\sqrt{1 - x^2},
\qquad
\xi_2(x) = \sqrt{1 - x^2}.
\]

These are continuous semialgebraic functions with $\xi_1(x) < \xi_2(x)$.
They partition the fiber into five cells:
\[
\begin{aligned}
S_{3,1} &= \{(x,y) : -1<x<1,\; y < \xi_1(x)\}, \\
S_{3,2} &= \{(x,y) : -1<x<1,\; y = \xi_1(x)\}, \\
S_{3,3} &= \{(x,y) : -1<x<1,\; \xi_1(x) < y < \xi_2(x)\}, \\
S_{3,4} &= \{(x,y) : -1<x<1,\; y = \xi_2(x)\}, \\
S_{3,5} &= \{(x,y) : -1<x<1,\; \xi_2(x) < y\}.
\end{aligned}
\]

\medskip
\noindent
The example above illustrates the key feature of cylindrical cell
decomposition: a definable set is partitioned into finitely many simple
regions, each described by continuous functions over lower-dimensional
cells.

In general, for definable sets in o-minimal structures (and in particular
in the restricted Pfaffian structure), such decompositions always exist.
Moreover, for our purposes, it is crucial that one can control the
\emph{number} and \emph{complexity} of the resulting cells in terms of
the complexity of the defining formulas.

\medskip

The quantitative theory of restricted Pfaffian structures provides such
control via the notions of \(*\)-format and \(*\)-degree for sub-Pfaffian
sets, introduced in \cite{binyamini2020effective}. These notions measure
definitional complexity in a way that is stable under the operations
appearing in cell decomposition (such as Boolean combinations,
projections, and lifting).

We will not recall the full definitions here, and instead rely only on the
following comparison with the usual notion of format and degree, together
with three key structural results.

\begin{remark}[Comparison with usual format and degree]
\label{rem:usual-to-star}
If a sub-Pfaffian set has format \(F\) and degree \(D\), then it has
\(*\)-format \(F\) and \(*\)-degree \(\Poly_F(D)\)
\cite[Remark~10]{BinyaminiNovikovZak2024WilkiePfaffian}.
\end{remark}

\begin{theorem}[Theorem 2 of \cite{binyamini2020effective}]
\label{thm:formula-to-star}
Let \(\varphi\) be a formula in \(\R_{\rPfaff}\), with \(*\)-format
\(F\) and \(*\)-degree \(D\). Then the set defined by \(\varphi\) has
\(*\)-format \(O_F(1)\) and \(*\)-degree \(\Poly_F(D)\).
\end{theorem}

\begin{theorem}[Sharp cell decomposition for \(\R_{\rPfaff}\)]
\label{thm:sharp-cd-rpfaff}
Let \(X_1,\ldots,X_k \subseteq \R^\ell\) be definable sets, each of
\(*\)-format at most \(F\) and \(*\)-degree at most \(D\).
Then there exists a cylindrical cell decomposition of \(\R^\ell\)
compatible with all \(X_i\), consisting of at most \(\Poly_F(D,k)\) cells.

Moreover, each cell is definable with \(*\)-format \(O_F(1)\) and
\(*\)-degree \(\Poly_F(D)\).
\end{theorem}

This result follows from
\cite[Theorem 1]{binyamini2020effective} together with
\cite[Remark 10]{BinyaminiNovikovZak2024WilkiePfaffian}.

\subsection{\texorpdfstring{Growth function of definable sets in $\R_{\rPfaff}$}{Growth function of definable sets in R rPfaff}}
\label{app:rPfaff-growth-function}

\begin{theorem}
\label{thm:sharp-shatter}
Let $\cH = \{h_a : a \in \R^k\}$ be a hypothesis class on $\R^l$
definable in $k$ parameters by a formula $\Phi_\cH(X, Y)$ in $\R_{\rPfaff}$
of format $F$ and degree $D$, where $X \in \R^l$ and $Y \in \R^k$.
Then
\[
  \Pi_{\cH}(m) \;\le\; \Poly_F(D) \cdot m^k.
\]
\end{theorem}

\begin{proof}
We adapt the argument of Johnson--Laskowski \cite[Lemma 4.2 and Proposition 4.3]{LaskowskiCompression} to $\R_{\rPfaff}$. The key difference is that, in place of the $i$-th boundary point of a fiber $\theta(\bar a, \R, \bar e)$ (which exists by o-minimality but with no control on its complexity), we use the threshold functions produced by sharp cell decomposition. This gives explicit bounds on the format and degree of every formula in the construction.

We write $\lg(\bar X)$ for the length of a tuple $\bar X$, i.e., the number of its coordinates.

\emph{From now on, we call $*$-format and $*$-degree just format and degree.}

\paragraph{Step 1: shaving one variable.}

Let $\theta(\bar X, Y, \bar Z)$ be a formula with $\bar X \in \R^l$, of format $O_F(1)$ and degree $\Poly_F(D)$. By \Cref{thm:sharp-cd-rpfaff}, applied to the singleton family defined by $\{(\bar x, y, \bar z) \mid \theta(\bar x, y, \bar z)\}$, we obtain a cylindrical cell decomposition compatible with $\theta$, with respect to the variable order $\bar Z, \bar X, Y$, so that $Y$ is the last variable. Let $\cS$ denote the resulting decomposition of the $(\lg(\bar Z) + l)$-dimensional space obtained by projecting away $Y$. Then each cell $C \in \cS$ has format $O_F(1)$ and degree $\Poly_F(D)$, and $|\cS| \le \Poly_F(D)$.

For each cell $C \in \cS$, the cells lying above $C$ in the full decomposition are described by continuous definable threshold functions
\[
  -\infty = \xi_0^C < \xi_1^C(\bar Z, \bar X) < \cdots < \xi_{N_C}^C(\bar Z, \bar X) < \xi_{N_C+1}^C = +\infty
\]
on $C$ with $N_C \le \Poly_F(D)$. Moreover, the graph of each function \(\xi_j^C\) is definable with format \(O_F(1)\) and degree \(\Poly_F(D)\). Equivalently, the formula expressing \( y=\xi_j^C(\bar Z,\bar X)\) has format \(O_F(1)\) and degree \(\Poly_F(D)\). In the formulas below, we use \(\xi_j^C\) by a slight abuse of notation, with this graph representation understood.

The cells over $C$ are the graphs $\{Y = \xi_j^C\}$ ($1 \le j \le N_C$) and the open strips $\{\xi_j^C < Y < \xi_{j+1}^C\}$ ($0 \le j \le N_C$). Since the decomposition is compatible with $\theta$, for any $(\bar z, \bar x) \in C$, the truth value of $\theta(\bar x, Y, \bar z)$ is invariant on each open strip, and is also constant on each graph $Y=\xi_j^C(\bar z,\bar x)$.

In analogy with Johnson--Laskowski, define the following formulas (in free variables $\bar X, \bar W, \bar Z$, where $\bar W$ plays the role of a parameter from $A$):
\begin{align*}
  \psi^C_{j, 1}(\bar X, \bar W, \bar Z) &:= (\bar Z, \bar W) \in C \,\land\, \theta(\bar X, \xi_j^C(\bar Z, \bar W), \bar Z), \\
  \psi^C_{j, 2}(\bar X, \bar W, \bar Z) &:= (\bar Z, \bar W) \in C \,\land\, \exists \varepsilon > 0 \;\forall U \, \big( \xi_j^C(\bar Z, \bar W) - \varepsilon < U < \xi_j^C(\bar Z, \bar W) \rightarrow \theta(\bar X, U, \bar Z) \big), \\
  \psi^C_{j, 3}(\bar X, \bar W, \bar Z) &:= (\bar Z, \bar W) \in C \,\land\, \exists \varepsilon > 0 \;\forall U \, \big( \xi_j^C(\bar Z, \bar W) < U < \xi_j^C(\bar Z, \bar W) + \varepsilon \rightarrow \theta(\bar X, U, \bar Z) \big), \\
  \psi^*(\bar X, \bar W, \bar Z) &:= \forall U \, \theta(\bar X, U, \bar Z).
\end{align*}
Let $\cF_\theta$ be the resulting finite collection of formulas. Each formula in $\cF_\theta$ has format $O_F(1)$ and degree $\Poly_F(D)$, and $|\cF_\theta| \le \Poly_F(D)$.

\smallskip\noindent\textit{Claim.}
For every finite $A \subseteq \R^l$, every $c \in \R$, and every $\bar e \in \R^{\lg(\bar Z)}$, there exist $\bar a \in A$ and $\psi \in \cF_\theta$ such that
\[
  \{\bar x\in A:\theta(\bar x, c, \bar e)\}
  =
  \{\bar x\in A:\psi(\bar x, \bar a, \bar e)\}.
\]

\smallskip

For each $\bar a \in A$, the pair $(\bar e, \bar a)$ lies in a unique cell $C(\bar e, \bar a) \in \cS$. Let
\[
  \Delta = \bigcup_{\bar a \in A} \{\xi_j^{C(\bar e, \bar a)}(\bar e, \bar a) : 1 \le j \le N_{C(\bar e, \bar a)}\},
\]
a finite subset of $\R$. There are four cases:

\textit{Case 1: $c \in \Delta$.} Then $c = \xi_j^C(\bar e, \bar a)$ for some $\bar a \in A, C \in \cS$ and some $j$, and
\[
  \{\bar x\in A:\theta(\bar x, c, \bar e)\}
  =
  \{\bar x\in A:\psi^C_{j, 1}(\bar x, \bar a, \bar e)\}.
\]

\textit{Case 2: $c \notin \Delta$, but $c < d$ for some $d \in \Delta$.}
Let $d^* \in \Delta$ be the least element greater than $c$, say $d^* = \xi_j^C(\bar e, \bar a)$ for some $\bar a \in A, C \in \cS$ and $j$.

Since $d^*$ is the least boundary value of $\Delta$ above $c$, there exists $\varepsilon > 0$ such that
\[
(d^* - \varepsilon, d^*) \subseteq (c, d^*)
\]
and this interval contains no element of $\Delta$.

Therefore, for every $\bar x \in A$, the truth of $\theta(\bar x, y, \bar e)$ is invariant for $y \in (d^* - \varepsilon, d^*)$, and agrees with its value at $c$.
Hence
\[
  \{\bar x\in A:\theta(\bar x, c, \bar e)\}
  =
  \{\bar x\in A:\psi^C_{j, 2}(\bar x, \bar a, \bar e)\}.
\]

\textit{Case 3: $\Delta \neq \emptyset$ but $c > d$ for every $d \in \Delta$.}
Let $d^* = \max \Delta$, say $d^* = \xi_j^C(\bar e, \bar a)$ for some $\bar a \in A, C \in \cS$.

Since there are no elements of $\Delta$ greater than $d^*$, there exists $\varepsilon > 0$ such that
\[
(d^*, d^* + \varepsilon)
\]
contains no element of $\Delta$.

Therefore, for every $\bar x \in A$, the truth of $\theta(\bar x, y, \bar e)$ is invariant for $y \in (d^*, d^* + \varepsilon)$, and agrees with its value at $c$.
Hence
\[
  \{\bar x\in A:\theta(\bar x, c, \bar e)\}
  =
  \{\bar x\in A:\psi^C_{j, 3}(\bar x, \bar a, \bar e)\}.
\]

\textit{Case 4: $\Delta = \emptyset$.} The truth of $\theta(\bar x, y, \bar e)$ is invariant on all of $\R$ for every $\bar x \in A$, and
\[
  \{\bar x\in A:\theta(\bar x, c, \bar e)\}
  =
  \{\bar x\in A:\psi^*(\bar x, \bar a, \bar e)\}
\]
for any choice of $\bar a \in A$.

\paragraph{Step 2: iterating.}
Let $\Phi_\cH(\bar X, \bar Y)$ with $\bar X \in \R^l$ and $\bar Y \in \R^k$ be the defining formula of $\cH$. Applying Step 1 to $\Phi_\cH$ shaves off the last coordinate $Y_k$, producing $\Poly_F(D)$ formulas in variables $\bar X, \bar W_k, Y_1, \ldots, Y_{k-1}$, each of format $O_F(1)$ and degree $\Poly_F(D)$. Iterating this construction $k$ times, peeling off one coordinate of $\bar Y$ at each step, we obtain a finite family $\cF_{\Phi_\cH}$ of formulas $\psi(\bar X, \bar W_1, \ldots, \bar W_k)$ such that for every finite $A \subseteq \R^l$ and every $\bar c \in \R^k$, there exist $(\bar a_1, \ldots, \bar a_k) \in A^k$ and $\psi \in \cF_{\Phi_\cH}$ with
\[
  \{\bar x\in A:\Phi_\cH(\bar x, \bar c)\}
  =
  \{\bar x\in A:\psi(\bar x, \bar a_1, \ldots, \bar a_k)\}.
\]
At each step the format remains $O_F(1)$ and the number of formulas is multiplied by $\Poly_F(D)$, so $|\cF_{\Phi_\cH}| \le \Poly_F(D)^k = \Poly_F(D)$ (absorbing $k \le F$ into the polynomial).

\paragraph{Step 3: counting traces.}
Every trace of $\cH$ on $A$ is represented by some tuple
\[
(\psi,\bar a_1,\ldots,\bar a_k)\in \cF_{\Phi_\cH}\times A^k.
\]
Therefore the number of traces is at most
\[
|\cF_{\Phi_\cH}|\cdot |A|^k \le \Poly_F(D)\cdot |A|^k. \qedhere
\]
\end{proof}

\subsection{\texorpdfstring{The embedding of \(\R_{\exp|_{[-M,M]}}\) into \(\R_{\rPfaff}\)}{The embedding of Rexp -M,M into RrPfaff}}
\label{app:embedding-exp-rPfaff}

The choice of the unit cube is only a normalization. If a restricted Pfaffian function is given on another compact box \(B\subseteq \R^n\), then after an affine change of variables \(B\simeq I^n\), it can be viewed as a restricted Pfaffian function on \(I^n\). Thus allowing compact boxes instead of \(I^n\) does not change the resulting structure.

For example, the restricted exponential structure
\(\R_{\exp|_{[-M,M]}}\) is embedded in \(\R_{\rPfaff}\).
Indeed, under the affine change of variables
\[
  x=2Mt-M,\qquad t\in[0,1],
\]
we replace \(\exp(x)\) by the normalized function
\[
  f_M(t):=\frac{\exp(2Mt-M)}{\exp M},
\]
which takes values in \([0,1]\).
Then \(f_M\colon [0,1]\to[0,1]\) extends to a Pfaffian function on a
neighborhood of \([0,1]\), since it satisfies the differential equation
\[
  f_M'(t)=2M\,f_M(t).
\]

Thus \(f_M\) is a Pfaffian function of order \(1\), defined by a polynomial
relation of degree \(1\), and in particular contributes \(O(1)\) to the
$*$-format and $*$-degree (uniformly in \(M\)).

Moreover, the original restricted exponential can be recovered via
\[
  \exp(x)=\exp(M)\, f_M(t),
  \qquad x=2Mt-M.
\]
Therefore every formula in \(\R_{\exp|_{[-M,M]}}\) can be interpreted as a
formula in \(\R_{\rPfaff}\), after applying the affine rescaling to the input
variables and the corresponding rescaling to the function values.

\paragraph{Syntactic complexity of formulas.}
For our purposes, it will be more convenient to work with a notion of
complexity defined directly at the level of formulas, rather than via
Pfaffian chains. We therefore introduce a syntactic definition of
format and degree for existential formulas in
\(\R_{\exp|_{[-M,M]}}\), which will serve as the primary measure of

complexity in what follows.

\begin{definition}[Format and degree in $\R_{\exp|_{[-M,M]}}$]
\label{def:format_degree_Rexp_restricted}
Let $\Phi(\bar Y)$ be a formula in $\R_{\exp|_{[-M,M]}}$.
We write $\Phi$ in \emph{existential normal form}
\[
  \Phi(\bar Y) \equiv \exists \bar W\,\theta(\bar Y,\bar W),
\]
where $\theta$ is a quantifier-free formula given as a Boolean combination
of atomic predicates of the following types:
\[
  P(\bar Y,\bar W)\ \sigma\ 0,
  \qquad \sigma\in\{=,>,<,\ge,\le\},
\]
and
\[
  U=\exp|_{[-M,M]}(V),
\]
where $P$ is a polynomial.

Let $n=|\bar Y|$, let $f=|\bar W|$, let $P_1,\dots,P_s$ be the
polynomials appearing in $\theta$, and let $r$ be the number of
constraints of the form $U=\exp|_{[-M,M]}(V)$.
The \emph{format} and \emph{degree} are defined by
\[
  F := n + f + r,
  \qquad
  D := \sum_{i=1}^s \deg(P_i) + r.
\]
The pair $(F,D)$ is called the \emph{complexity} of $\Phi$.
\end{definition}

\paragraph{Comparison of complexities.}
We compare the syntactic complexity $(F,D)$ of formulas in
$\R_{\exp|_{[-M,M]}}$ (Definition~\ref{def:format_degree_Rexp_restricted})
with the $*$-format and $*$-degree used in the restricted Pfaffian
framework. These notions are not identical: the former is a syntactic
measure for existential formulas, while the latter records the structure
of Pfaffian chains and the degrees of the associated polynomial relations.
We only need a one-sided comparison.

Let $\Phi$ be a formula in $\R_{\exp|_{[-M,M]}}$ of complexity $(F,D)$.
After the normalization described above, each restricted exponential
constraint
\[
  U=\exp|_{[-M,M]}(V)
\]
is replaced by a restricted Pfaffian function. If several such constraints
appear, their Pfaffian chains can be merged into a single chain, increasing
the Pfaffian format by at most $r\le F$. The polynomial atomic predicates
remain unchanged, with total degree bounded by $D$.

Therefore, after interpreting $\Phi$ as a formula in $\R_{\rPfaff}$, the
formula has $*$-format $O_F(1)$ and $*$-degree $\Poly_F(D)$, uniformly in
$M$. By \Cref{thm:formula-to-star}, the set defined by $\Phi$ has
$*$-format $O_F(1)$ and $*$-degree $\Poly_F(D)$.

Equivalently, every set definable by an existential
$\R_{\exp|_{[-M,M]}}$-formula of complexity $(F,D)$ admits a restricted
sub-Pfaffian presentation of $*$-format $O_F(1)$ and
$*$-degree $\Poly_F(D)$.

\subsection{Proof of Main Theorem~\ref{main:bounds_R_exp}}
\label{app:proof-main-three}

Recall the notion of format and degree for existential formulas
(Definition~\ref{def:format_degree_Rexp_restricted}).  

In the case of the full exponential structure \(\R_{\exp}\), the
definition is identical, except that the exponential constraints are of
the form
\[
  U=\exp(V)
\]
instead of \(U=\exp|_{[-M,M]}(V)\).

Once again for a tuple \(X\), we write \(\lg(X)\) for its length (number of coordinates).

\begin{theorem}
\label{thm:shatter_Rexp}
Let \(\cH=\{h_a:a\in \R^k\}\) be a hypothesis class on \(\R^l\), definable in
\(k\) parameters by an existential formula \(\Phi(X,A)\) in \(\R_{\exp}\) of
format \(F\) and degree \(D\), where \(\lg(X)=l\) and \(\lg(A)=k\). Then
\[
  \Pi_{\cH}(m) \le \Poly_F(D)\cdot m^k .
\]
\end{theorem}

\begin{proof}
Let \(\cH\) be defined by an existential formula
\[
  \Phi(X,A) \equiv \exists W\,\theta(X,W,A),
\]
where
\[
  X=(X_1,\ldots,X_l),\qquad
  A=(A_1,\ldots,A_k),\qquad
  W=(W_1,\ldots,W_f).
\]
Assume that this formula has format and degree \((F,D)\) in the sense of
Section~\ref{sec:main-results}. We also assume that \(\theta\) is written in
graph form, namely as a Boolean combination of atomic predicates of the following
two types:
\begin{enumerate}
    \item Polynomial equalities and inequalities
    \[
      P(X,W,A)\ \sigma\ 0,
      \qquad
      \sigma\in\{=,>,<,\ge,\le\},
    \]
    where \(P\) is a real polynomial.

    \item Exponential graph constraints
    \[
      U=\exp(V),
    \]
    where \(U,V\) are variables among \(X,W,A\).
\end{enumerate}

For \(M>0\), let \(\Phi_M(X,A)\) be the formula in
\(\R_{\exp|_{[-M,M]}}\) obtained from \(\Phi(X,A)\) by replacing each
occurrence of \(\exp\) by \(\exp|_{[-M,M]}\). Define
\[
  \Psi_M(X,A)
  :=
  \exists W\;
  \Bigl[
  \theta_M(X,W,A)
  \wedge
  \bigwedge_{j=1}^l |X_j|\le M
  \wedge
  \bigwedge_{j=1}^k |A_j|\le M
  \wedge
  \bigwedge_{j=1}^f |W_j|\le M
  \Bigr],
\]
where \(\theta_M\) is obtained from \(\theta\) by replacing each occurrence
of \(\exp\) by \(\exp|_{[-M,M]}\).

Each formula \(\Phi_M(X,A)\) has \(*\)-format \(O_F(1)\) and \(*\)-degree
\(\Poly_F(D)\). Passing from \(\Phi_M\) to \(\Psi_M\) adds only boundedness
constraints and does not introduce new Pfaffian structure. Hence, by the
inductive definition of \(*\)-format and \(*\)-degree, the \(*\)-format
increases by at most a constant, and the \(*\)-degree increases by at most
\(O(l+k+f)\le O(F)\).

Therefore, by \Cref{thm:sharp-shatter},
\[
  \sup_{M\in\N} \Pi_{\Psi_M}(m)
  \le
  \Poly_F(D)\cdot m^k .
\]

We now show that this uniform bound implies the desired bound for \(\Phi\).
Fix arbitrary points
\[
  x_1,\ldots,x_m\in\R^l .
\]
For every sign pattern on these points realized by \(\Phi\), choose one
parameter vector \(a\in\R^k\) realizing it. Let \(A_0\subseteq\R^k\) be the
finite set of all chosen parameter vectors.

Choose \(M\) large enough so that all coordinates of all points \(x_i\) and all
parameters \(a\in A_0\) are bounded in absolute value by \(M\). We now enlarge
\(M\) if necessary as follows. For every pair \((x_i,a)\) with \(a\in A_0\), if
\(\Phi(x_i,a)\) holds, choose a witness \(w\in\R^f\) such that
\[
  \theta(x_i,w,a)
\]
holds, and enlarge \(M\) so that all coordinates of this witness are also
bounded in absolute value by \(M\). Since there are only finitely many such
pairs, this gives a finite value \(M_m\).

We claim that \(\Psi_{M_m}\) realizes on \(x_1,\ldots,x_m\) every sign pattern
realized by \(\Phi\). Indeed, fix \(a\in A_0\) and \(x_i\).

If \(\Phi(x_i,a)\) holds, then by construction there is a witness
\(w\in[-M_m,M_m]^f\) such that \(\theta(x_i,w,a)\) holds. Since all variables
appearing in the exponential graph constraints are bounded by \(M_m\), the
restricted exponential agrees with the ordinary exponential on this witness.
Thus \(\Psi_{M_m}(x_i,a)\) also holds.

Conversely, if \(\Psi_{M_m}(x_i,a)\) holds, then there is a witness
\(w\in[-M_m,M_m]^f\) such that the restricted formula holds. Again, since all
variables appearing in the exponential graph constraints are bounded by
\(M_m\), the restricted exponential agrees with the ordinary exponential.
Hence \(\Phi(x_i,a)\) holds.

Therefore \(\Phi\) and \(\Psi_{M_m}\) realize the same labels on \(x_1,\ldots,x_m\) for every \(a\in A_0\). In particular,
\[
  \Pi_{\Phi}(m)
  \le
  \Pi_{\Psi_{M_m}}(m).
\]
Using the uniform bound on \(\Pi_{\Psi_M}(m)\), we obtain
\[
  \Pi_{\Phi}(m)
  \le
  \Poly_F(D)\cdot m^k .
\]
This proves the theorem.
\end{proof}

We are now ready to finish the proof of Main Theorem~\ref{main:bounds_R_exp}.

\begin{restatedtheorem}{main:bounds_R_exp}
Let~$\mathcal H$ be a hypothesis class over $\mathbb R^l$ definable with $k$ parameters in the language of $\mathbb R_{\exp}$, and let $\mathcal N$ be a neighborhood system definable in the language of $\mathbb R_{\exp}$. Suppose, moreover, that both $\mathcal H$ and $\mathcal N$ are defined using existential formulas of format $F$ and degree $D$. Then there is a universal
function \(\gamma\colon \N\to\N\), computable explicitly, such that
\[
  \Pi_{\cH^\cN}(m) \le D^{\gamma(F)} m^k .
\]

Consequently,
\[
  m_{\mathrm{ERM}}^{\mathrm{real}}(\varepsilon,\delta)=  O\Bigl(\frac{k\log(1/\varepsilon) + \gamma(F) \log(D) + \log(1/\delta)}{\varepsilon}\Bigr)
\]
and
\[
\VCdim(\mathcal H^\mathcal N)  = O_F(\log D).
\]
\end{restatedtheorem}

\begin{proof}
Let \(\Phi_{\cH}(Y,A)\) be an existential formula defining the hypothesis class
\(\cH=\{h_a:a\in\R^k\}\), so that
\[
  h_a(Y)=1 \quad\Longleftrightarrow\quad \Phi_{\cH}(Y,a),
\]
and let \(\Phi_\cN(X,Y)\) be an existential formula defining the neighborhood relation \(\cN\). Then the transformed class \(\cH^\cN\) is defined by
\[
  \Phi_{\cH^\cN}(X,A)
  :=
  \exists Y\,
  \bigl(
    \Phi_\cN(X,Y)\wedge \Phi_{\cH}(Y,A)
  \bigr).
\]
Indeed, \(h_a^{\cN}(X)=1\) holds precisely when there is some \(Y\) reachable from \(X\) through \(\cN\) such that \(h_a(Y) = 1\).

Since both \(\Phi_{\cH}\) and \(\Phi_\cN\) have complexity \((F,D)\), the formula
\(\Phi_{\cH^\cN}\) is again existential, with format at most \(2F\) and degree at
most \(2D\). Here we only add one extra existential block and conjoin the two
quantifier-free parts; the total number of variables and the total polynomial
degree are bounded by the sum of the two original bounds.

Applying Theorem~\ref{thm:shatter_Rexp} to the hypothesis class \(\cH^\cN\), whose parameter space still has dimension \(k\), gives
\[
  \Pi_{\cH^\cN}(m)
  \le
  \Poly_F(D)\,m^k .
\]
Since the polynomial \(\Poly_F(D)\) has degree and coefficients depending only
on \(F\), after increasing \(\gamma(F)\) we may write
\[
  \Poly_F(D)\le D^{\gamma(F)}
\]
with the convention that \(D\ge 2\). Thus
\[
  \Pi_{\cH^\cN}(m)\le D^{\gamma(F)}m^k .
\]

The sample-complexity bound now follows directly from
Lemma~\ref{lem:erm-from-growth}:
\[
m_{\mathrm{ERM}}^{\mathrm{real}}(\varepsilon,\delta)
=
O\Bigl(
  \frac{
    k\log(1/\varepsilon)
    + \gamma(F)\log D
    + \log(1/\delta)
  }{\varepsilon}
\Bigr).
\]

Finally, the VC-dimension bound follows from
Lemma~\ref{lem:vc-from-growth}. Indeed, we have shown that
\[
  \Pi_{\cH^\cN}(m)\le D^{\gamma(F)}m^k
\]
for every \(m\ge 1\). Applying Lemma~\ref{lem:vc-from-growth} with
\[
  C=D^{\gamma(F)}
\]
gives
\[
  \VCdim(\cH^\cN)
  \le
  2\gamma(F)\log D + 4k\log(4k).
\]
Since \(k\le F\), the second term is bounded by a constant depending only on
\(F\). Hence
\[
  \VCdim(\cH^\cN)=O_F(\log D). \qedhere
\]
\end{proof}

\section{Beyond \texorpdfstring{$\R_{\exp}$}{Rexp}}
\label{app:beyond-Rexp}

The proof of Main Theorem~\ref{main:bounds_R_exp}, given in
Appendix~\ref{app:sec:proof-main:bounds_R_exp}, does not rely on any
specific properties of the exponential function. The only
property used is that, after restricting all variables to a bounded
domain and applying an affine rescaling, the functions appearing in the
defining formulas become \emph{restricted Pfaffian functions} with
controlled complexity.
In particular, it suffices that every function in the language is
\emph{Pfaffian on the whole real line}. 

Let
\[
  \R_g := (\R,+,\cdot,0,1,<,g),
\]
where \(g\colon \R\to\R\) is a Pfaffian function on the whole real line.
Assume that \(g\) has Pfaffian format \(F_\circ\) and Pfaffian degree
\(D_\circ\), in the sense of Definition~\ref{def:pfaff}.

We assign format and degree to existential formulas in the language of
\(\R_g\) as follows. This definition is very similar to the corresponding definition for
\(\R_{\exp}\), given in Definition~\ref{def:format_degree_Rexp_restricted}; the only difference is that
we must also account for the format and degree of the function \(g\) itself.
Let
\[
  \Phi(X) \equiv \exists Y\,\theta(X,Y)
\]
be an existential formula, where \(\theta\) is quantifier-free and is written as
a Boolean combination of atomic predicates of the forms
\(P(X,Y)\,\sigma\,0\), with \(\sigma\in\{=,>,<,\ge,\le\}\), and \(z=g(w)\).
Here \(P\) is a polynomial in the variables \(X,Y\), and in the atomic formulas of the form \(z=g(w)\), both \(z\) and \(w\) are variables drawn from \(X \cup Y\).

The \emph{format} of \(\Phi\) is defined as
\[
  |X|+|Y|+rF_\circ,
\]
where \(r\) is the number of atomic predicates of the form \(z=g(w)\).
The \emph{degree} of \(\Phi\) is defined as
\[
  \sum_i \deg(P_i)+rD_\circ,
\]
where the sum is over all polynomials \(P_i\) appearing in polynomial atomic predicates of \(\theta\).

\begin{theorem}
\label{thm:beyond-Rexp}
Let \(\mathcal H=\{h_a:a\in\R^k\}\) be a hypothesis class on \(\R^\ell\), and let
\(\mathcal N\) be a neighborhood system on \(\R^\ell\). Suppose that both
\(\mathcal H\) and \(\mathcal N\) are definable in \(\R_g\) by existential
formulas of format at most \(F\) and degree at most \(D\). Then
\[
  m_{\mathrm{ERM}}^{\mathrm{real}}(\varepsilon,\delta)
  =
  O\Bigl(
    \frac{
      k\log(1/\varepsilon)
      + \gamma(F)\log D
      + \log(1/\delta)
    }{\varepsilon}
  \Bigr),
\]
and
\[
  \VCdim(\mathcal H^{\mathcal N})=O_F(\log D).
\]
Here \(\gamma\colon\N\to\N\) is the same explicitly computable function as in
Main Theorem~\ref{main:bounds_R_exp}.
\end{theorem}

\begin{proof}
The proof is the same as the proof of Main Theorem~\ref{main:bounds_R_exp}.
The only additional point is the comparison between formulas in \(\R_g\) and
restricted Pfaffian formulas.

Fix \(M>0\). On the interval \([-M,M]\), write
\[
  x=2Mt-M,\qquad t\in[0,1].
\]
The restriction of \(g\) to \([-M,M]\) is represented on the unit interval by
the function
\[
  \widetilde g_M(t):=g(2Mt-M).
\]
We claim that \(\widetilde g_M\) has Pfaffian format and degree bounded in
terms of \(F_\circ\) and \(D_\circ\), uniformly in \(M\).

Indeed, suppose that \(g\) is defined with respect to a Pfaffian chain
\[
  f_1,\ldots,f_r
\]
on \(\R\), with
\[
  \frac{d f_i}{dx}(x)
  =
  P_i\bigl(x,f_1(x),\ldots,f_i(x)\bigr).
\]
Define
\[
  \widetilde f_i(t):=f_i(2Mt-M).
\]
Then
\[
  \frac{d\widetilde f_i}{dt}(t)
  =
  2M\,
  P_i\bigl(2Mt-M,\widetilde f_1(t),\ldots,\widetilde f_i(t)\bigr).
\]
Thus \(\widetilde f_1,\ldots,\widetilde f_r\) is again a Pfaffian chain.
The affine change of variables changes only the coefficients of the defining
polynomials, not their degrees, and therefore does not change the relevant
Pfaffian format or degree bounds. The same applies to \(\widetilde g_M\).

Consequently, every formula in \(\R_g\), restricted to a box \([-M,M]^n\), can
be interpreted as a formula in \(\R_{\rPfaff}\) with \(*\)-format
\(O_F(1)\) and \(*\)-degree \(\Poly_F(D)\), uniformly in \(M\). The argument of Appendix~\ref{app:proof-main-three} then applies
with only minor modifications: use the
growth-function bound for the restricted formulas via Theorem~\ref{thm:sharp-shatter}, and then pass to the
unrestricted formula by choosing \(M\) large enough for any fixed finite sample
and the finitely many witnesses needed on that sample.

Thus
\[
  \Pi_{\mathcal H^\cN}(m)\le D^{\gamma(F)}m^k.
\]
The stated ERM sample-complexity and VC-dimension bounds follow from
Lemma~\ref{lem:erm-from-growth} and Lemma~\ref{lem:vc-from-growth}, exactly as
in the proof of Main Theorem~\ref{main:bounds_R_exp}.
\end{proof}

\begin{remark}
The same argument applies if the language is expanded by finitely many globally
defined Pfaffian functions, possibly of several variables. One only has to add
the corresponding Pfaffian formats and degrees to the definitions of format and
degree above. After restricting to a compact box and applying an affine
rescaling to the unit cube, the resulting functions are again restricted
Pfaffian functions with the same degree bounds, up to changes in the
coefficients. We omit this straightforward generalization.
\end{remark}

\section{Technical lemmas}
\label{app:technical-lemmas}

Throughout this appendix, logarithms are base \(2\).

\begin{restatedlemma}{lem:erm-from-growth}
Let $\mathcal H\subseteq\{0,1\}^{\mathcal X}$, and suppose that
\(\Pi_{\mathcal H}(m)\le C m^k\) for every \(m\ge 1\), where
\(C\ge 1\) and \(k\ge 1\). Then
\[
m_{\mathrm{ERM}}^{\mathrm{real}}(\varepsilon,\delta)
=
O\left(
\frac{
k\log(k/\varepsilon)+\log C+\log(1/\delta)
}{\varepsilon}
\right).
\]
\end{restatedlemma}

\begin{proof}
We use the standard realizable ERM bound
\[
  \Pr\bigl[L_D(h_S)>\varepsilon\bigr]
  \le
  \Pi_{\mathcal H}(2m)e^{-\varepsilon m/2}.
\]
By the growth assumption, \(\Pi_{\mathcal H}(2m)\le C(2m)^k\).
Thus it is enough to choose \(m\) so that
\[
  C(2m)^k e^{-\varepsilon m/2}\le \delta,
  \qquad\text{equivalently}\qquad
  \frac{\varepsilon m}{2}
  \ge
  \log C+k\log(2m)+\log(1/\delta).
\]

Set
\[
  x:=\frac{\varepsilon m}{2},
  \qquad
  L:=\log C+\log(1/\delta).
\]
Since \(2m=4x/\varepsilon\), the sufficient condition becomes
\[
  x
  \ge
  L+k\log\left(\frac{4x}{\varepsilon}\right)
  =
  L+k\log(4/\varepsilon)+k\log x.
\]
Equivalently, it is enough to solve an inequality of the form
\(x\ge a+k\log x\), where
\[
  a:=L+k\log(4/\varepsilon).
\]
By Lemma~\ref{lem:log-self-bound}, it is enough to take
\[
  x
  =
  O\bigl(a+k\log k\bigr)
  =
  O\bigl(
  \log C+\log(1/\delta)+k\log(1/\varepsilon)+k\log k
  \bigr).
\]
Since \(x=\varepsilon m/2\), this gives
\[
m
=
O\left(
\frac{
k\log(k/\varepsilon)+\log C+\log(1/\delta)
}{\varepsilon}
\right).
\]
\end{proof}

\begin{lemma}[VC dimension from polynomial growth]
\label{lem:vc-from-growth}
Let $\mathcal H\subseteq\{0,1\}^{\mathcal X}$, and suppose that
\(\Pi_{\mathcal H}(m)\le C m^k\) for every \(m\ge 1\), where
\(C\ge 1\) and \(k\ge 1\). Then
\[
  \VCdim(\mathcal H)
  \le
  2\log C + 4k\log(4k).
\]
\end{lemma}

\begin{proof}
Let \(d=\VCdim(\mathcal H)\). If \(d=0\), there is nothing to prove.
Otherwise, by the definition of VC dimension and by the growth assumption,
\[
  2^d\le \Pi_{\mathcal H}(d)\le C d^k.
\]
Taking logarithms gives
\[
  d\le \log C+k\log d.
\]
Set \(a:=\log C\) and \(b:=k\). Then \(d\le a+b\log d\), so
Lemma~\ref{lem:log-self-bound} gives
\[
  d\le 2a+4b\log(4b)
  =
  2\log C+4k\log(4k).
\]
Thus \(\VCdim(\mathcal H)\le 2\log C+4k\log(4k)\).
\end{proof}

\begin{lemma}[Logarithmic self-bound]
\label{lem:log-self-bound}
Let \(x\ge 1\), \(a\ge 0\), and \(b\ge 1\). If
\[
  x\le a+b\log x,
\]
then
\[
  x\le 2a+4b\log(4b).
\]
\end{lemma}

\begin{proof}
If \(x\le 2a\), then immediately
\[
  x\le 2a\le 2a+4b\log(4b).
\]
It remains to consider the case \(x>2a\). Then \(a<x/2\), and therefore
\[
  x\le a+b\log x
  <
  \frac{x}{2}+b\log x.
\]
Hence \(x<2b\log x\). Set \(B:=2b\). Since \(b\ge 1\), we have
\(B\ge 2\), and the last inequality becomes
\[
  x<B\log x.
\]

We now claim that \(x\le 2B\log(2B)\). Suppose, toward a contradiction,
that \(x>2B\log(2B)\). Since \(B\ge 2\), we have
\(2B\log(2B)\ge 4\log 4>e\), and hence \(x>e\). The function
\(t\mapsto \log t/t\) is decreasing for all \(t\ge e\), so
\[
  \frac{\log x}{x}
  <
  \frac{\log(2B\log(2B))}{2B\log(2B)}.
\]
Also,
\[
  \log(2B\log(2B))\le 2\log(2B),
\]
because this is equivalent to
\[
  2B\log(2B)\le (2B)^2,
\]
or equivalently \(\log(2B)\le 2B\), which follows from the elementary
inequality \(\log u\le u\) for \(u>0\). Therefore
\[
  \frac{\log x}{x}
  <
  \frac{2\log(2B)}{2B\log(2B)}
  =
  \frac1B.
\]
Multiplying by \(B x\), we get \(B\log x<x\), contradicting
\(x<B\log x\). Hence \(x\le 2B\log(2B)\).

Substituting \(B=2b\), we obtain
\[
  x\le 4b\log(4b).
\]
Together with the first case \(x\le 2a\), this gives
\[
  x\le 2a+4b\log(4b). \qedhere
\]
\end{proof}

\begin{lemma}
\label{lemma:vc-consistency}
Let \(k\ge 1\), \(d\ge 1\), and \(A\ge 2\). Suppose that
\[
  d \le k\log(Ad/k),
\]
where logarithms are to base \(2\). Then
\[
d \le 4k\log A.
\]
\end{lemma}

\begin{proof}
We consider two cases.

First, suppose that
\[
  A \ge d/k .
\]
Then
\[
  d \le k\log(Ad/k)
  \le k\log(A^2)
  =
  2k\log A .
\]

Otherwise,
\[
  A < d/k .
\]
Then
\[
  d \le k\log(Ad/k)
  <
  k\log\bigl((d/k)^2\bigr)
  =
  2k\log(d/k).
\]
Equivalently, with \(t=d/k\),
\[
  t < 2\log t .
\]
This implies \(t<4\), and hence \(d<4k\). Since \(A\ge 2\), we have
\[
  4k \le 4k\log A .
\]

\end{proof}

\section*{Acknowledgements}

We would like to thank Steve Hanneke for insightful discussions.

\bibliographystyle{alpha}
\bibliography{refs}

\newpage

\appendix

\end{document}